\documentclass{article}
\usepackage[utf8]{inputenc}
\usepackage{fullpage}

\PassOptionsToPackage{numbers}{natbib}

\usepackage[T1]{fontenc}    %
\usepackage{hyperref}       %
\usepackage{url}            %
\usepackage{booktabs}       %
\usepackage{amsfonts}       %
\usepackage{nicefrac}       %
\usepackage{microtype}      %
\usepackage{xcolor}         %
\usepackage{enumitem}
\usepackage{tikz}
\usepackage{adjustbox}
\usepackage{subcaption}
\usepackage{graphicx}
\usepackage{wrapfig}
\usepackage{natbib}

\usetikzlibrary{patterns}
\usetikzlibrary{calc}
\usetikzlibrary{angles,quotes}
\usetikzlibrary{shapes.misc}
\usetikzlibrary{shapes.geometric}
\usetikzlibrary{arrows.meta,arrows}

\ExplSyntaxOn
\cs_new:Npn \expandableinput #1
  { \use:c { @@input } { \file_full_name:n {#1} } }
\ExplSyntaxOff

\usepackage{mathtools}      %
\usepackage{amssymb}
\usepackage{amsthm}         %
\usepackage{mathdots}       %
\usepackage{amsfonts}       %
\usepackage{nicefrac}       %
\usepackage{physics}        %
\usepackage{bm}             %
\usepackage{centernot}      %
\usepackage{thmtools}       %
\usepackage{dsfont}
\usepackage{bbold}
\usepackage{thmtools,thm-restate}
\usepackage{xspace}
\usepackage{xcolor}
\usepackage{empheq}

\newtheorem{definition}{Definition}
\newtheorem{lemma}{Lemma}

\newtheorem{corollary}{Corollary}

\DeclareMathOperator*{\essinf}    {ess\,inf}
\DeclareMathOperator*{\Vol}       {Vol}
\DeclareMathOperator*{\SphereCap} {Cap}

\DeclareMathOperator*{\Ebb}       {\mathbb{E}}
\DeclareMathOperator*{\Pbb}       {\mathbb{P}}

\DeclareMathOperator*{\NC}{NC}
\DeclareMathOperator*{\NCP}{NCP}
\DeclareMathOperator*{\PC}{PC}

\DeclareMathOperator*{\Span}{Span}

\newcommand*\diff{\,\mathop{}\!\mathrm{d}}

\definecolor{myblue}{HTML}{457b9d}
\definecolor{myred}{HTML}{E63946}
\definecolor{mygreen}{HTML}{40916c}
\definecolor{myyellow}{HTML}{ffbe0b}

\definecolor{myorange}{RGB}{247,214,157}
\definecolor{mygray}{RGB}{211,211,211}
\definecolor{mypurple}{RGB}{212,164,217}
\definecolor{gray3}{gray}{0.5}
\definecolor{gray2}{gray}{0.7}
\definecolor{gray1}{gray}{0.85}

\DeclareRobustCommand{\tikzcircle}[1]{%
  \mathrel{%
    \tikz[baseline=-0.5ex]{%
      \draw[#1,fill=#1,radius=0.1] (0,0) circle;%
    }%
  }%
}

\newcommand{\ie}{\emph{i.e.}\xspace}

\newcommand{\leftmatrix}       {\begin{pmatrix}}
\newcommand{\rightmatrix}      {\end{pmatrix}}
\newcommand{\leftmatrixsmall}  {\begin{psmallmatrix}}
\newcommand{\rightmatrixsmall} {\end{psmallmatrix}}

\def\ddefloop#1{\ifx\ddefloop#1\else\ddef{#1}\expandafter\ddefloop\fi}

\def\ddef#1{\expandafter\def\csname #1bb\endcsname{\ensuremath{\mathbb{#1}}}}
\ddefloop ABCDEFGHIJKLMNOPQRSTUVWXYZ\ddefloop

\def\ddef#1{\expandafter\def\csname #1cal\endcsname{\ensuremath{\mathcal{#1}}}}
\ddefloop ABCDEFGHIJKLMNOPQRSTUVWXYZ\ddefloop

\def\ddef#1{\expandafter\def\csname #1mat\endcsname{\ensuremath{\mathbf{#1}}}}
\ddefloop ABCDEFGHIJKLMNOPQRSTUVWXYZ\ddefloop

\def\ddef#1{\expandafter\def\csname #1matsf\endcsname{\ensuremath{\mathsf{#1}}}}
\ddefloop ABCDEFGHIJKLMNOPQRSTUVWXYZ\ddefloop

\def\ddef#1{\expandafter\def\csname #1vec\endcsname{\ensuremath{\mathbf{#1}}}}
\ddefloop abcdefghijklmnopqrstuvwxyz\ddefloop

\usepackage{todonotes}

\usepackage{hyperref}
\usepackage{cleveref}

\title{Towards Evading the Limits of Randomized Smoothing: \\ A Theoretical Analysis}

\author{
    Raphael Ettedgui$^{*1}$, Alexandre Araujo$^{*2}$, Rafael Pinot$^{3}$, Yann Chevaleyre$^{1}$, Jamal Atif$^{1}$ \\[0.3cm]
    $^{1}$ LAMSADE, Université Paris-Dauphine, CNRS, PSL University \\
    $^{2}$ INRIA, Ecole Normale Supérieure, CNRS, PSL University \\
    $^{3}$ DCL, School of Computer and Communication Sciences, EPFL, Switzerland
}

\date{}

\makeatletter
\def\blfootnote{\gdef\@thefnmark{}\@footnotetext}
\makeatother

\begin{document}
\maketitle

\blfootnote{$^*$ Equal contributions}

\begin{abstract}
Randomized smoothing is the dominant standard for provable defenses against adversarial examples. Nevertheless, this method has recently been proven to suffer from important information theoretic limitations. In this paper, we argue that these limitations are not intrinsic, but merely a byproduct of current certification methods. We first show that these certificates use too little information about the classifier, and are in particular blind to the local curvature of the decision boundary. This leads to severely sub-optimal robustness guarantees as the dimension of the problem increases. We then show that it is theoretically possible to bypass this issue by collecting more information about the classifier. More precisely, we show that it is possible to approximate the optimal certificate with arbitrary precision, by probing the decision boundary with several noise distributions. Since this process is executed at certification time rather than at test time, it entails no loss in natural accuracy while enhancing the quality of the certificates. This result fosters further research on classifier-specific certification and demonstrates that randomized smoothing is still worth investigating. Although classifier-specific certification may induce more computational cost, we also provide some theoretical insight on how to mitigate it.
\end{abstract}

\section{Introduction}

Modern day machine learning models are vulnerable to adversarial examples, \ie, small perturbations to their inputs that force misclassification~\cite{globerson2006nightmare,szegedy2013intriguing,goodfellow2014explaining}.
Although a considerable amount of work has been devoted to mitigating their impact~\cite{goodfellow2014explaining,kurakin2016adversarial,madry2018towards,carlini2017towards,wong2018provable,weng2018towards,athalye2018obfuscated,raghunathan2018certified,raghunathan2018semidefinite,pinot2019theoretical,pinot2020randomization,araujo2020advocating,araujo2021lipschitz,meunier2022dynamical,li2019preventing,trockman2021orthogonalizing,singla2021skew,singla2021householder,gehr2018ai2,xu2020automatic,tjeng2018evaluating}, defending a model against these malicious inputs remains very challenging.
This difficulty to successfully resist adversarial attacks essentially comes from the difficulty to evaluate the defense mechanisms. In fact, several works have shown that it is of paramount importance to evaluate the reliability of defense mechanisms even in the worst-case scenario, \ie, against a strong and adaptive attack~\cite{tramer2020adaptive,carlini2019evaluating}.
Given that it is in general intractable to build an optimal attack for a given classifier, the community now favors \emph{certified defenses} that provide provable robustness guarantees.
Among existing certification techniques, \emph{randomized smoothing}, first introduced in~\cite{lecuyer2018certified} and refined in~\cite{li2019certified, cohen2019certified,salman2019provably}, has emerged as a dominant standard. In short, randomized smoothing applies a convolution between a base classifier and a Gaussian distribution to enhance the robustness of the classifier. 
This technique has the advantage of offering provable robustness guarantees. Formally, for any given input, we can guarantee that within a radius of this input, the classifier cannot change prediction. The maximal size of this radius is called the \emph{certification radius} of the point. Beyond Gaussian, this technique has been generalized and can now be studied under a large class of smoothing distributions~\cite{yang2020randomized}.
In short, randomized smoothing is considered the state-of-the-art for provable robustness for multiple classification tasks, among which image classification and segmentation~\cite{fischer2021scalable}.

\textbf{Limitations of Randomized Smoothing.}
However, this scheme does not provide certification for free.
In fact, it presents a trade-off in that using a distribution with a higher variance leads to more robust models, but at the cost of a reduction in standard accuracy. Evaluating the maximum accuracy that a certified method can achieve for a target certified radius is therefore crucial to assess the viability of this method. In this respect, several recent works support the argument that for a given probability distribution, the maximum radius that can be certified vanishes as the dimension of the learning problem increases~\cite{blum2020random,hayes2020extensions,kumar2020curse,yang2020randomized,mohapatra2021hidden,wu2021completing}. Specifically, 
\citet{blum2020random} have shown that to successfully certify high-dimensional images, a \emph{$\ell_p$-smoothing distributions} must have such a large variance that it leads to a trivial classifier. This seems to make the problem of simultaneously achieving robustness and accuracy unfeasible. Building upon this first observation
\citet{hayes2020extensions} demonstrated a similar result for \emph{location-scale distribution} (which include the generalized Gaussian distributions) but their use of the KL divergence makes the bound loose. A tighter bound has recently been devised by
\citet{kumar2020curse} showed that the largest $\ell_p$-certification radius for some specific class of smoothing distributions decreases with rate $\mathcal{O}(1/d^{\frac{1}{2} - \frac{1}{p}})$ where $d$ is the dimension of the images. This result was further generalized by
\citet{yang2020randomized}. Along the same line, \citet{wu2021completing} demonstrated an equivalent result for $p > 2$ for \emph{spherical symmetric distributions} and a tight bound for $\ell_2$-certification radius.
From an information-theoretic perspective, these results suggest that random smoothing is a doomed method for high-dimensional learning problems. This may lead some in the community to consider this method irrelevant in many modern machine learning tasks. 

\textbf{Closely Related Work.}
Two recent works \cite{dvijotham2020framework,mohapatra2020higherorder} have started exploring the lead of new certification methods, by using additional information on the classifier. For instance,
\citet{dvijotham2020framework} use the full probability distribution at each point instead of just the probability of the dominant class. They combine this approach with an alternative to the Neyman-Pearson certification, namely relaxing the attack constraint using $f$-divergences, and manage to obtain slightly better certificates than existing certification methods. Similarly, \citet{mohapatra2020higherorder} and \citet{levine2021tight} showed that using first-order or second-order information leads to better certified radius.
\citeauthor{mohapatra2020higherorder} also shows that it is theoretically possible to reconstruct a Gaussian smoothed classifier using only information about its successive derivatives at the point of interest (even the first derivatives are, however, extremely expensive to compute).

\textbf{Our contributions.}
In this paper, we provide a theoretical analysis of the certification process for randomized smoothing. We advocate that current limitations are not intrinsic to the scheme, but a byproduct of current certification methods. To do so, we first focus on the uniform distribution, and provide a framework to dissociate the information-gathering process (and so the certification) from the smoothing itself, whereas current certificates use the same noise distribution for both. This makes it possible to increase the quality of the certificates without affecting the standard accuracy. Our main findings can be summarized as follows.
\begin{enumerate}[parsep=0pt,itemsep=3pt,topsep=0pt,leftmargin=15pt]
  \item To exhibit the role of the local curvature on the quality of the certificate, we focus our study on two types of decision boundaries: cones of revolution and 2-piecewise linear. For these two types of boundaries, we quantify the gap between current certificates and the optimal one that randomized smoothing could provide. This validates our hypothesis by showing that the steeper the local curvature, the more suboptimal the certification. We also see that the gap becomes wider as the dimension grows, thus explaining the current limitations.
  \item Secondly, we show that the generalized Neyman-Pearson Lemma can be used to improve the quality of certificates by collecting information from several noise distributions at the same time. By separating this information-gathering step from the smoothing itself, better certification radius can be obtained without any further loss in standard accuracy, thus circumventing the current information theoretic limitations.
  \item Finally, we show that using this framework, it is possible to approximate the perfect certificate, for any decision boundaries, with arbitrary precision using information from a finite number of distributions.
  Although our proof requires a number of noises that increases exponentially with the dimension, we show that it is possible to drastically reduce the number of noises required when prior information is available on the classifier.
  We also provide several insights into the certification design process, especially computational issues, by providing techniques for reducing the dimension of the Neyman-Pearson set. More precisely, we show that it is possible to compute certificates without sampling in high-dimension by combining uniform and Gaussian distribution and leveraging the isotropic properties of the latter.
\end{enumerate}

\section{General framework for Randomized Smoothing}
\label{section:framework}

Let $\mathcal{X} = \mathbb{R}^d$ be our input space and $\Ycal = \{ 0, 1 \}$ our label space.
Let $\mathcal{H}$ be the class of measurable functions from $\Xcal$ to $\Ycal$, and $h\in \mathcal{H}$ be a base classifier.
Randomized smoothing creates a new classifier $h_{q_0}$ by averaging $h$ under some probability density function $q_0$ over $\mathcal{X}$. When receiving an input $x\in \mathcal{X}$, we compute the probability that $h$ takes value 1 for a point drawn from $q_0(\ \cdot\  - x)$:
\begin{equation*}
  p(x, h, q_0) = \int h(z) q_0(z-x) \diff z
\end{equation*}
The smoothed classifier then returns the most probable class.
\begin{definition}
  The $q_0$-randomized smoothing of $h$ is the classifier:
  \begin{equation*}
  h_{q_0} : x \mapsto \mathbb{1} \left\{ p(x, h, q_0) > \frac{1}{2} \right\}
  \end{equation*}
\end{definition}
In the rest of the paper, we will consider the points $x$ such that $p(x,h,q_0) > \frac{1}{2}$, so where the smoothed classifier returns 1.
The other case is exactly symmetrical.
An adversarial attack $\delta \in \mathcal{X}$ is a small crafted perturbation, such that $\norm{\delta} \leq \epsilon$ where $\epsilon$ is a small constant and $\norm{\ \cdot \ }$ is the Euclidean norm.
An adversarial attack is engineered such that:
\begin{equation*}
  h_{q_0}(x + \delta) \neq y, 
\end{equation*}
meaning $p(x + \delta, h, q_0) \leq \frac{1}{2}$.
In the following, we will define the robustness guarantee provided by randomized smoothing.

\begin{definition}[\textbf{$\epsilon$-certificate}]
An $\epsilon$-certificate for the $q_0$-randomized smoothing of $h$ at point $x$ is any lower bound on the probability after attack, i.e., some value $v \in \Rbb$ such that:
\begin{equation*}
    v \leq \inf\limits_{\delta\in B(0,\epsilon)} p(x+\delta, h, q_0)
\end{equation*}
A certificate $v$ is said to be successful if $v > \frac{1}{2}$.
\end{definition}

A successful $\epsilon$-certificate means that no attack of norm at most $\epsilon$ can fool the classifier.
This definition allows us to compare different certificates for the same smoothed classifier.

Certificates for randomized smoothing are usually ``black-box'', \ie we can only access the classifier $h$ through limited queries. This means giving a bound on the worst-case scenario for some class of functions $\mathcal{G}$ that we know contains $h$.

\begin{definition}[\textbf{Noised-based certificate}]
\label{definition:noised_based_certificate}
Let $\mathcal{Q}$ be a finite family of probability density functions.
Let $q_0 \in \mathcal{Q}$. The $\mathcal{Q}-$noise-based $\epsilon$-certificate for the $q_0$-randomized smoothing of $h$ at point $x$ is: 
\begin{equation*}
  \NC(h, q_0, x,\epsilon, \mathcal{Q}) = \inf\limits_{g \in \mathcal{G}_{\mathcal{Q}}} \inf\limits_{\delta \in B(0,\epsilon)} p(x+\delta, g, q_0) 
\end{equation*}
where:
\begin{equation*}
\mathcal{G}_{\mathcal{Q}}=\left\{ g \in \mathcal{H} \mid \forall q \in \mathcal{Q},\ p(x,g,q) = p(x,h,q) \right\}
\end{equation*}
\end{definition}

A noise-based certificate is a lower bound over all classifiers that exhibit the same response as $h$ to every noise distribution in $\mathcal{Q}$.
The certificate from \citet{cohen2019certified} is a particular type of noise-based certificate, where we only use one distribution to gather information, namely the same $q_0$ that is used for the smoothing, \ie, $\mathcal{Q} = \left\{ q_0\right\}$.

Note that there is a fundamental difference between $q_0$, the noise used for the smoothing, which is a part of the smoothed classifier $h_{q_0}$ used at test time, and the family $\mathcal{Q}$, which are noises used to analyze the base classifier $h$. Noises from $\mathcal{Q}$ are only used for information-gathering.

To evaluate the quality of this certificate, we now need a benchmark. For that, we will use the \textit{perfect certificate}, \ie, the tightest possible bound, that uses full information over the classifier $h$.

\begin{definition}[{\bf Perfect certificate}] The perfect 
\makebox{$\epsilon$-certificate} for the $q_0$-randomized smoothing of $h$ at point $x$ is:
\begin{equation*}
  \PC(h, q_0, x,\epsilon) = \inf\limits_{\delta\in B(0,\epsilon)} p(x+\delta, h, q_0)
\end{equation*}
\end{definition}

The underestimation between prefect certificates and noise-based certificates can now be defined as the difference between both bounds.
\begin{definition}[{\bf Underestimation of a noise-based certificate}]
\label{definition:diff_pc_nc}
Let $\mathcal{Q}$ be a finite family of probability density functions and let $q_0 \in \mathcal{Q}$ and $\epsilon > 0$. We define the underestimation function $\nu$ as:
\begin{equation*}
  \nu(h, q_0, x, \epsilon, \mathcal{Q}) = \PC(h, q_0, x, \epsilon) - \NC(h, q_0, x, \epsilon, \mathcal{Q})
\end{equation*}
\end{definition}
The function $\nu$ computes the difference between the perfect $\epsilon$-certificate and the noise-based $\epsilon$-certificate for an classifier $h$ with randomized smoothing $q_0$.

\section{Limitations of current certificates}
\label{section:limitations_rs}

In this section, we provide insight on the limitation of randomized smoothing.
Recall that single-noise certificates, \ie, $\NC(h, q_0, x, \epsilon, \{q_0\})$, use the same noise $q_0$ for smoothing and information-gathering.
However, this technique presents several weaknesses:
\begin{enumerate}[topsep=0pt,parsep=0pt,leftmargin=12pt]
  \item Since $\Qcal$ is small, the certificate is obtained as a worst-case over a large set of functions $\Gcal$. This will often make it significantly poorer than the optimal certificate $\PC$ for our specific classifier.
  \item With this kind of certificate, we have limited possibilities of optimization for the choice of the base classifier $h$. In particular, current certificates are blind to the ``local curvature'' of the decision boundary, as will be illustrated shortly.
  \item Since the distribution used for the smoothing and information-gathering operations is the same. This means that a larger variance leads to more information on the decision frontier (and thus to better certificates) but at the cost of a loss of standard accuracy.
\end{enumerate}

\subsection{Theoretical analysis with toy decision boundaries}

To illustrate these limitations, we provide a deeper analysis of the underestimation function defined in~\Cref{definition:diff_pc_nc}.
In the following, we focus on the case of uniform noise distributions on an $\ell_2$ ball.

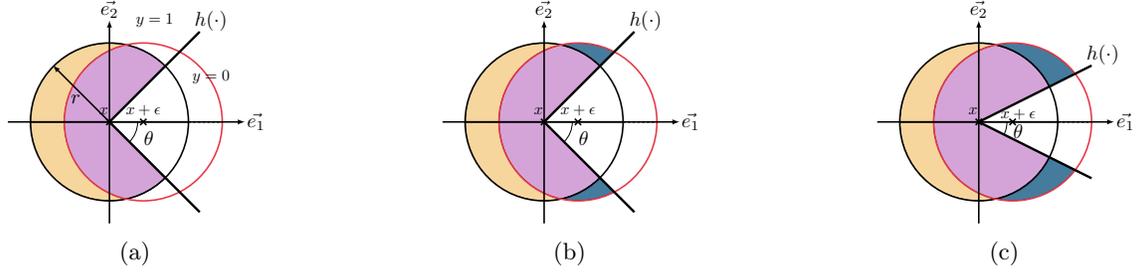
\begin{figure}[t]
  \centering
  \begin{subfigure}[h]{0.30\textwidth}
    \centering
    \scalebox{0.3}{%
      \def\centerarc[#1](#2)(#3:#4:#5)%
    { \draw[#1] ($(#2)+({#5*cos(#3)},{#5*sin(#3)})$) arc (#3:#4:#5); }
\tikzset{%
  >={Latex[width=0mm,length=0mm]},
  base/.style = {fill=white, font=\sffamily},
  dot/.style = {circle, fill=black, minimum size=#1, inner sep=0pt, outer sep=0pt},
  cross/.style = {cross out, draw=black, minimum size=4pt, inner sep=0pt, outer sep=0pt}
}
\begin{tikzpicture}[every node/.style=base, align=center, scale=5]

  \clip (-0.4, -0.4) rectangle (1.9, 1.6);

  \begin{scope}
    \clip (0.5, 0.5) -- (0.5+0.8, 0.5+0.8) -- (-0.5, 1.5) -- (-0.5, -0.5) -- (0.5+0.8, 0.5-0.8) -- (0.5, 0.5);
    \clip (0.5, 0.5) circle (0.7);
    \draw[fill=myorange,even odd rule] (0.5, 0.5) circle (0.7) (0.8, 0.5) circle (0.7);
  \end{scope}

  \begin{scope}
    \clip (0.5, 0.5) -- (0.5+0.8, 0.5+0.8) -- (-0.5, 1.5) -- (-0.5, -0.5) -- (0.5+0.8, 0.5-0.8) -- (0.5, 0.5);
    \clip (0.8, 0.5) circle (0.7);
    \draw[fill=mypurple,even odd rule] (0.5, 0.5) circle (0.7);
  \end{scope}

  \draw[color=black, line width=2pt] (0.5, 0.5) circle (0.7);
  \draw[color=myred, line width=2pt] (0.8, 0.5) circle (0.7);

  \draw[color=black, line width=2pt] (-0.2, 0.5) -- (1.2, 0.5);
  \draw[color=black, line width=2pt, dashed] (1.2, 0.5) -- (1.5, 0.5);

  \draw (.5,.5) node[cross,rotate=0,scale=2,line width=2] {};
  \draw (.8,.5) node[cross,rotate=0,scale=2,line width=2] {};

  \draw[color=black, line width=2pt, -latex] (-0.4, +0.5) -- (+1.7, +0.5);
  \draw[color=black, line width=2pt, -latex] (+0.5, -0.4) -- (+0.5, +1.4);
  \node[draw=none, fill=white, text opacity=1, fill opacity=0., scale=2.5] at (1.8, 0.5) {$\vec{e_1}$};
  \node[draw=none, fill=white, text opacity=1, fill opacity=0., scale=2.5] at (0.5, 1.5) {$\vec{e_2}$};

  \draw[color=black, line width=2pt, -latex] (+0.5, +0.5) -- (+0.0, +1.0);
  \node[draw=none, fill=white, text opacity=1, fill opacity=0., scale=2.5] at (0.2, 0.7) {$r$};

  \node[draw=none, fill=white, text opacity=1, fill opacity=0., scale=2.5] at (0.85,  0.36) {$\theta$};
  \node[draw=none, fill=white, text opacity=1, fill opacity=0., scale=2.0] at (0.45,  0.6) {$x$};

  \node[draw=none, fill=white, text opacity=1, fill opacity=0., scale=2] at (1.4, 0.9) {$y = 0$};
  \node[draw=none, fill=white, text opacity=1, fill opacity=0., scale=2] at (0.9, 1.4) {$y = 1$};
  \node[draw=none, fill=white, text opacity=1, fill opacity=0., scale=2.0] at (0.45+0.35,  0.6) {$x+\epsilon$};

  \draw[color=black, line width=3pt] (0.5, 0.5) -- (0.5+0.8, 0.5+0.8);
  \draw[color=black, line width=3pt] (0.5, 0.5) -- (0.5+0.8, 0.5-0.8);
  \node[draw=none, fill=white, text opacity=1, fill opacity=0., scale=2.5] at (1.4, 1.4) {$h(\cdot)$};

  \centerarc[black](0.5, 0.5)(0:-45:0.25);

\end{tikzpicture}%
    }%
    \label{figure:underestimate_cohen_certificate_a}
    \caption{}
  \end{subfigure}
   \hfill
  \begin{subfigure}[h]{0.30\textwidth}
    \centering
    \scalebox{0.3}{%
      \def\centerarc[#1](#2)(#3:#4:#5)%
    { \draw[#1] ($(#2)+({#5*cos(#3)},{#5*sin(#3)})$) arc (#3:#4:#5); }
\tikzset{%
  >={Latex[width=0mm,length=0mm]},
  base/.style = {fill=white, font=\sffamily},
  dot/.style = {circle, fill=black, minimum size=#1, inner sep=0pt, outer sep=0pt},
  cross/.style = {cross out, draw=black, minimum size=4pt, inner sep=0pt, outer sep=0pt}
}
\begin{tikzpicture}[every node/.style=base, align=center, scale=5]

  \clip (-0.4, -0.4) rectangle (1.9, 1.6);

  \begin{scope}
    \clip (0.5, 0.5) -- (0.5+0.8, 0.5+0.8) -- (-0.5, 1.5) -- (-0.5, -0.5) -- (0.5+0.8, 0.5-0.8) -- (0.5, 0.5);
    \clip (0.5, 0.5) circle (0.7);
    \draw[fill=myorange,even odd rule] (0.5, 0.5) circle (0.7) (0.8, 0.5) circle (0.7);
  \end{scope}

  \begin{scope}
    \clip (0.5, 0.5) -- (0.5+0.8, 0.5+0.8) -- (-0.5, 1.5) -- (-0.5, -0.5) -- (0.5+0.8, 0.5-0.8) -- (0.5, 0.5);
    \clip (0.8, 0.5) circle (0.7);
    \draw[fill=mypurple,even odd rule] (0.5, 0.5) circle (0.7);
  \end{scope}

  \begin{scope}
    \clip (0.5, 0.5) -- (0.5+0.8, 0.5+0.8) -- (0., 1.5) -- (0, -0.5) -- (0.5+0.8, 0.5-0.8) -- (0.5, 0.5);
    \clip (0.8, 0.5) circle (0.7); 
    \draw[fill=myblue,even odd rule] (0.8, 0.5) circle (0.7) (0.5, 0.5) circle (0.7);
  \end{scope}

  \draw[color=black, line width=2pt] (0.5, 0.5) circle (0.7);
  \draw[color=myred, line width=2pt] (0.8, 0.5) circle (0.7);

  \draw[color=black, line width=2pt] (-0.2, 0.5) -- (1.2, 0.5);
  \draw[color=black, line width=2pt, dashed] (1.2, 0.5) -- (1.5, 0.5);

  \draw (.5,.5) node[cross,rotate=0,scale=2,line width=2] {};
  \draw (.8,.5) node[cross,rotate=0,scale=2,line width=2] {};

  \draw[color=black, line width=2pt, -latex] (-0.4, +0.5) -- (+1.7, +0.5);
  \draw[color=black, line width=2pt, -latex] (+0.5, -0.4) -- (+0.5, +1.4);
  \node[draw=none, fill=white, text opacity=1, fill opacity=0., scale=2.5] at (1.8, 0.5) {$\vec{e_1}$};
  \node[draw=none, fill=white, text opacity=1, fill opacity=0., scale=2.5] at (0.5, 1.5) {$\vec{e_2}$};

  \node[draw=none, fill=white, text opacity=1, fill opacity=0., scale=2.5] at (0.85,  0.36) {$\theta$};
  \node[draw=none, fill=white, text opacity=1, fill opacity=0., scale=2.0] at (0.45,  0.6) {$x$};
  \node[draw=none, fill=white, text opacity=1, fill opacity=0., scale=2.0] at (0.45+0.35,  0.6) {$x+\epsilon$};

  \draw[color=black, line width=3pt] (0.5, 0.5) -- (0.5+0.8, 0.5+0.8);
  \draw[color=black, line width=3pt] (0.5, 0.5) -- (0.5+0.8, 0.5-0.8);
  \node[draw=none, fill=white, text opacity=1, fill opacity=0., scale=2.5] at (1.4, 1.4) {$h(\cdot)$};

  \centerarc[black](0.5, 0.5)(0:-45:0.25);

\end{tikzpicture}%
    }%
    \label{figure:underestimate_cohen_certificate_b}
    \caption{}
  \end{subfigure}
  \hfill
  \begin{subfigure}[h]{0.30\textwidth}
    \centering
    \scalebox{0.3}{%
      \def\centerarc[#1](#2)(#3:#4:#5)%
    { \draw[#1] ($(#2)+({#5*cos(#3)},{#5*sin(#3)})$) arc (#3:#4:#5); }
\tikzset{%
  >={Latex[width=0mm,length=0mm]},
  base/.style = {fill=white, font=\sffamily},
  dot/.style = {circle, fill=black, minimum size=#1, inner sep=0pt, outer sep=0pt},
  cross/.style = {cross out, draw=black, minimum size=4pt, inner sep=0pt, outer sep=0pt}
}
\begin{tikzpicture}[every node/.style=base, align=center, scale=5]

  \clip (-0.4, -0.4) rectangle (1.9, 1.6);

  \begin{scope}
    \clip (0.5, 0.5) -- (0.5+0.8, 0.5+0.8) -- (-0.5, 1.5) -- (-0.5, -0.5) -- (0.5+0.8, 0.5-0.8) -- (0.5, 0.5);
    \clip (0.5, 0.5) circle (0.7);
    \draw[fill=myorange,even odd rule] (0.5, 0.5) circle (0.7) (0.8, 0.5) circle (0.7);
  \end{scope}

  \begin{scope}
    \clip (0.5, 0.5) -- (0.5+1.0, 0.5+0.8-0.3) -- (-0.5, 1.5) -- (-0.5, -0.5) -- (0.5+1.0, 0.5-0.8+0.3) -- (0.5, 0.5);
    \clip (0.8, 0.5) circle (0.7);
    \draw[fill=mypurple,even odd rule] (0.5, 0.5) circle (0.7);
  \end{scope}

  \begin{scope}
    \clip (0.5, 0.5) -- (0.5+1.0, 0.5+0.8-0.3) -- (0., 1.5) -- (0, -0.5) -- (0.5+1.0, 0.5-0.8+0.3) -- (0.5, 0.5);
    \clip (0.8, 0.5) circle (0.7); 
    \draw[fill=myblue,even odd rule] (0.8, 0.5) circle (0.7) (0.5, 0.5) circle (0.7) ;
  \end{scope}
   
  \draw[color=black, line width=2pt] (0.5, 0.5) circle (0.7);
  \draw[color=myred, line width=2pt] (0.8, 0.5) circle (0.7);

  \draw[color=black, line width=2pt] (-0.2, 0.5) -- (1.2, 0.5);
  \draw[color=black, line width=2pt, dashed] (1.2, 0.5) -- (1.5, 0.5);

  \draw (.5,.5) node[cross,rotate=0,scale=2,line width=2] {};
  \draw (.8,.5) node[cross,rotate=0,scale=2,line width=2] {};

  \draw[color=black, line width=2pt, -latex] (-0.4, +0.5) -- (+1.7, +0.5);
  \draw[color=black, line width=2pt, -latex] (+0.5, -0.4) -- (+0.5, +1.4);
  \node[draw=none, fill=white, text opacity=1, fill opacity=0., scale=2.5] at (1.8, 0.5) {$\vec{e_1}$};
  \node[draw=none, fill=white, text opacity=1, fill opacity=0., scale=2.5] at (0.5, 1.5) {$\vec{e_2}$};

  \node[draw=none, fill=white, text opacity=1, fill opacity=0., scale=2.5] at (0.85,  0.42) {$\theta$};
  \node[draw=none, fill=white, text opacity=1, fill opacity=0., scale=2.0] at (0.45,  0.6) {$x$};
  \node[draw=none, fill=white, text opacity=1, fill opacity=0., scale=2.0] at (0.50+0.35, 0.57) {$x+\epsilon$};

  \draw[color=black, line width=3pt] (0.5, 0.5) -- (0.5+1.0, 0.5+0.8-0.3);
  \draw[color=black, line width=3pt] (0.5, 0.5) -- (0.5+1.0, 0.5-0.8+0.3);
  \node[draw=none, fill=white, text opacity=1, fill opacity=0., scale=2.5] at (1.6, 1.1) {$h(\cdot)$};

  \centerarc[black](0.5, 0.5)(0:-27:0.25);

\end{tikzpicture}%
    }%
    \label{figure:underestimate_cohen_certificate_c}
    \caption{}
  \end{subfigure}
  \caption{Illustration of \Cref{theorem:underestimate_cohen}.
  Figure (a) describes the probability $p(x, h, q_r) = q_r(\tikzcircle{myorange}) + q_r(\tikzcircle{mypurple})$.
  Figure (b) describes the perfect certificate as $\PC(h, q_r, x, \epsilon) = q_r(\tikzcircle{mypurple}) + q_r(\tikzcircle{myblue})$ whereas the single noised-based certificate is $\NC(h, q_r, x, \epsilon, \{q_r\}) = q_r(\tikzcircle{mypurple})$. 
  Figure (c) shows that blue zone increases with $\theta$, also, for $\theta = 0$, we have $q_r(\tikzcircle{myblue}) \xrightarrow[d \to \infty]{} 1$.}
  \label{figure:underestimate_cohen_certificate}
  \vspace{-0.3cm}
\end{figure}

{\em \textbf{Intuition of~\Cref{theorem:underestimate_cohen}.} For both conical and 2-piecewise-linear decision boundaries, the gap between uniform single-noise certificates and the uniform perfect certificate increases with the local curvature of the decision boundary. For high local curvatures, this gap becomes arbitrarily large as the dimension of the problem increases.
\Cref{figure:underestimate_cohen_certificate} illustrate this result in 2 dimensions.
}

\begin{restatable}[{\bf Underestimation of single noise-based certificates}]{theorem}{underestimateprop}
\label{theorem:underestimate_cohen} Let $\epsilon, r \in \Rbb_+^*$ such that $\epsilon \leq r$.
Let $\mathcal{Q} = \{q_0\}$ where $q_0$ is a uniform distribution over an $\ell_2$ ball $B_2^d(0,r)$. We denote $\theta_m = \arccos(\frac{\epsilon}{2r})$. For any $\theta \in \left[ 0, \theta_{m} \right]$, we denote by $h_\theta$ the classifier whose decision boundary is a cone of revolution of peak $0$, axis $e_1$ and angle $\theta$ where $(e_1,\dots,e_d)$ be any orthonormal basis of $\Rbb^d$.
Then, $\nu(h_\theta, q_0, 0, \epsilon, \mathcal{Q})$ is a continuous and decreasing function of $\theta$. Furthermore, we have
\begin{itemize}[itemsep=0pt,topsep=0pt,parsep=0pt]
    \item $\nu(h_0, q_0, 0, \epsilon, \mathcal{Q}) = 1 - I_{1-(\frac{\epsilon}{2r})^2} \left( \frac{d+1}{2}, \frac{1}{2} \right)$
    \item $\nu(h_{\theta_m}, q_0, 0, \epsilon, \mathcal{Q}) = 0$
\end{itemize}
where $I_z(a, b)$ is the incomplete regularized beta function.
Furthermore, for any $\epsilon, r$, $\nu(h_0, q_0, 0, \epsilon, \mathcal{Q}) \xrightarrow[d \to \infty]{} 1$.
The same result holds for 2-piecewise linear sets.
\end{restatable}

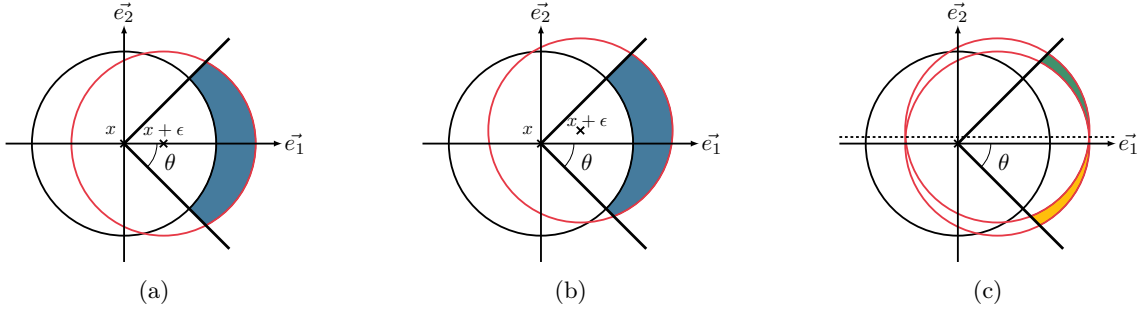
\begin{figure}[th]
  \centering
  \hfill
  \begin{subfigure}[h]{0.32\textwidth}
    \centering
    \scalebox{0.35}{%
      \def\centerarc[#1](#2)(#3:#4:#5)%
    { \draw[#1] ($(#2)+({#5*cos(#3)},{#5*sin(#3)})$) arc (#3:#4:#5); }
\tikzset{%
  >={Latex[width=0mm,length=0mm]},
  base/.style = {fill=white, font=\sffamily},
  dot/.style = {circle, fill=black, minimum size=#1, inner sep=0pt, outer sep=0pt},
  cross/.style = {cross out, draw=black, minimum size=4pt, inner sep=0pt, outer sep=0pt}
}
\begin{tikzpicture}[every node/.style=base, align=center, scale=5]

  \clip (-0.4, -0.4) rectangle (1.9, 1.6);

  \begin{scope}
    \clip (0.5, 0.5) -- (0.5+0.8, 0.5+0.8) -- (0.5+1.2, 0.5+0.8) -- (0.5+1.2, 0.5-0.8) -- (0.5+0.8, 0.5-0.8) -- (0.5, 0.5);
    \draw[fill=myblue, even odd rule] (0.8, 0.5) circle (0.7) (0.5, 0.5) circle (0.7) ;
  \end{scope}

  \draw[color=black, line width=2pt] (0.5, 0.5) circle (0.7);
  \draw[color=myred, line width=2pt] (0.8, 0.5) circle (0.7);

  \draw[color=black, line width=2pt, -latex] (-0.4, +0.5) -- (+1.7, +0.5);
  \draw[color=black, line width=2pt, -latex] (+0.5, -0.4) -- (+0.5, +1.4);

  \draw (0.5, 0.5) node[cross,rotate=0,scale=2,line width=2] {};
  \draw (0.8, 0.5) node[cross,rotate=0,scale=2,line width=2] {};
 
  \node[draw=none, fill=white, text opacity=1, fill opacity=0., scale=2.5] at (1.8, 0.5) {$\vec{e_1}$};
  \node[draw=none, fill=white, text opacity=1, fill opacity=0., scale=2.5] at (0.5, 1.5) {$\vec{e_2}$};
  \node[draw=none, fill=white, text opacity=1, fill opacity=0., scale=2.5] at (0.85,  0.36) {$\theta$};

  \node[draw=none, fill=white, text opacity=1, fill opacity=0., scale=2.0] at (0.40,  0.6) {$x$};
  \node[draw=none, fill=white, text opacity=1, fill opacity=0., scale=2.0] at (0.45+0.35,  0.6) {$x+\epsilon$};

  \draw[color=black, line width=3pt] (0.5, 0.5) -- (0.5+0.8, 0.5+0.8);
  \draw[color=black, line width=3pt] (0.5, 0.5) -- (0.5+0.8, 0.5-0.8);

  \centerarc[black](0.5, 0.5)(0:-45:0.25);

\end{tikzpicture}%
    }%
    \caption{}
  \end{subfigure}
  \hfill
  \begin{subfigure}[h]{0.32\textwidth}
    \centering
    \scalebox{0.35}{%
      \def\centerarc[#1](#2)(#3:#4:#5)%
    { \draw[#1] ($(#2)+({#5*cos(#3)},{#5*sin(#3)})$) arc (#3:#4:#5); }
\tikzset{%
  >={Latex[width=0mm,length=0mm]},
  base/.style = {fill=white, font=\sffamily},
  dot/.style = {circle, fill=black, minimum size=#1, inner sep=0pt, outer sep=0pt},
  cross/.style = {cross out, draw=black, minimum size=4pt, inner sep=0pt, outer sep=0pt}
}
\begin{tikzpicture}[every node/.style=base, align=center, scale=5]

  \clip (-0.4, -0.4) rectangle (1.9, 1.6);

  \begin{scope}
    \clip (0.5, 0.5) -- (0.5+0.8, 0.5+0.8) -- (0.5+1.2, 0.5+0.8) -- (0.5+1.2, 0.5-0.8) -- (0.5+0.8, 0.5-0.8) -- (0.5, 0.5);
    \draw[fill=myblue, even odd rule] (0.8, 0.6) circle (0.7) (0.5, 0.5) circle (0.7) ;
  \end{scope}
   
  \draw[color=black, line width=2pt] (0.5, 0.5) circle (0.7);
  \draw[color=myred, line width=2pt] (0.8, 0.6) circle (0.7);

  \draw[color=black, line width=2pt, -latex] (-0.4, +0.5) -- (+1.7, +0.5);
  \draw[color=black, line width=2pt, -latex] (+0.5, -0.4) -- (+0.5, +1.4);

  \draw (0.5, 0.5) node[cross,rotate=0,scale=2,line width=2] {};
  \draw (0.8, 0.6) node[cross,rotate=0,scale=2,line width=2] {};

  \node[draw=none, fill=white, text opacity=1, fill opacity=0., scale=2.5] at (1.8, 0.5) {$\vec{e_1}$};
  \node[draw=none, fill=white, text opacity=1, fill opacity=0., scale=2.5] at (0.5, 1.5) {$\vec{e_2}$};
  \node[draw=none, fill=white, text opacity=1, fill opacity=0., scale=2.5] at (0.85,  0.36) {$\theta$};

  \node[draw=none, fill=white, text opacity=1, fill opacity=0., scale=2.0] at (0.40,  0.6) {$x$};
  \node[draw=none, fill=white, text opacity=1, fill opacity=0., scale=2.0] at (0.45+0.40,  0.67) {$x+\epsilon$};

  \draw[color=black, line width=3pt] (0.5, 0.5) -- (0.5+0.8, 0.5+0.8);
  \draw[color=black, line width=3pt] (0.5, 0.5) -- (0.5+0.8, 0.5-0.8);

  \centerarc[black](0.5, 0.5)(0:-45:0.25);

\end{tikzpicture}%
    }%
    \caption{}
    \label{figure:lemma_fig2_paper}
  \end{subfigure}
  \hfill
  \begin{subfigure}[h]{0.32\textwidth}
    \centering
    \scalebox{0.35}{%
      \def\centerarc[#1](#2)(#3:#4:#5)%
    { \draw[#1] ($(#2)+({#5*cos(#3)},{#5*sin(#3)})$) arc (#3:#4:#5); }
\tikzset{%
  >={Latex[width=0mm,length=0mm]},
  base/.style = {fill=white, font=\sffamily},
  dot/.style = {circle, fill=black, minimum size=#1, inner sep=0pt, outer sep=0pt},
  cross/.style = {cross out, draw=black, minimum size=4pt, inner sep=0pt, outer sep=0pt}
}
\begin{tikzpicture}[every node/.style=base, align=center, scale=5]

  \clip (-0.4, -0.4) rectangle (1.9, 1.6);

  \begin{scope}
    \clip (0.5, 0.5) -- (0.5+0.8, 0.5+0.8) -- (0.5+1.2, 0.5+0.8) -- (0.5+1.2, 0.5) -- (0.5, 0.5);
    \draw[fill=mygreen, even odd rule] (0.8, 0.6) circle (0.7) (0.8, 0.5) circle (0.7) ;
  \end{scope}

  \begin{scope}
    \clip (0.5, 0.5) -- (0.5+0.8, 0.5-0.8) -- (0.5+1.2, 0.5-0.8) -- (0.5+1.2, 0.5) -- (0.5, 0.5);
    \draw[fill=myyellow, even odd rule] (0.8, 0.6) circle (0.7) (0.8, 0.5) circle (0.7) ;
  \end{scope}
   
  \draw[color=black, line width=2pt] (0.5, 0.5) circle (0.7);
  \draw[color=myred, line width=2pt] (0.8, 0.6) circle (0.7);
  \draw[color=myred, line width=2pt] (0.8, 0.5) circle (0.7);

  \draw[color=black, line width=2pt, -latex] (-0.4, +0.5) -- (+1.7, +0.5);
  \draw[color=black, line width=2pt, -latex] (+0.5, -0.4) -- (+0.5, +1.4);

  \draw (.5,.5) node[cross,rotate=0,scale=2,line width=2] {};

  \node[draw=none, fill=white, text opacity=1, fill opacity=0., scale=2.5] at (1.8, 0.5) {$\vec{e_1}$};
  \node[draw=none, fill=white, text opacity=1, fill opacity=0., scale=2.5] at (0.5, 1.5) {$\vec{e_2}$};
  \node[draw=none, fill=white, text opacity=1, fill opacity=0., scale=2.5] at (0.85,  0.36) {$\theta$};

  \draw[color=black, line width=2pt, dashed] (-0.4, +0.5+0.05) -- (+1.7, +0.5+0.05);

  \draw[color=black, line width=3pt] (0.5, 0.5) -- (0.5+0.8, 0.5+0.8);
  \draw[color=black, line width=3pt] (0.5, 0.5) -- (0.5+0.8, 0.5-0.8);

  \centerarc[black](0.5, 0.5)(0:-45:0.25);

\end{tikzpicture}%
    }%
    \caption{}
    \label{figure:lemma_fig3_paper}
  \end{subfigure}
  \caption{Illustration of the optimal attack with a cone of revolution as decision boundary. The optimal attack of norm $\epsilon$ is the vector $\delta = [\epsilon \vec{e_1}, 0, \dots, 0]$. Figure (a) shows that there is always a gain by translating along $e_1$, Figure (b) shows the gain when translating along both $e_1$ and $e_2$, and finally, Figure (c) shows the difference. The loss incurred by the second translation, visible in the yellow zone, is greater than the gain (green zone). The argument of the proof is that the reflection of the blue zone through the dotted hyperplane is contained in the yellow zone.}
  \label{figure:optimal_attac_proof_paper}
\end{figure}

\paragraph{Sketch of proof of~\Cref{theorem:underestimate_cohen}.}
First, in order to define the perfect certificate, we show that the optimal attack against a conical decision boundary is the translation along its axis. This means that the attack defined by $\delta = [ \epsilon \vec{e_1}, 0, \dots, 0 ]^\top$ is optimal.
To prove that, we exploit the symmetry of the problem, as illustrated in
\Cref{figure:optimal_attac_proof_paper}.
To compute the difference between the perfect certificate and single-noise certificate, we here again used the rotational symmetry of the problem around axis $e_1$, to compute the volume in {\em hyper-cylindrical coordinates} as defined \Cref{definition:hyper_cylindrical_coordinates}.
$\square$

\begin{wrapfigure}{r}{0.40\textwidth}
    \centering
    \vspace{-0.5cm}
    \hspace{-1.0cm}
    \includegraphics[scale=0.45]{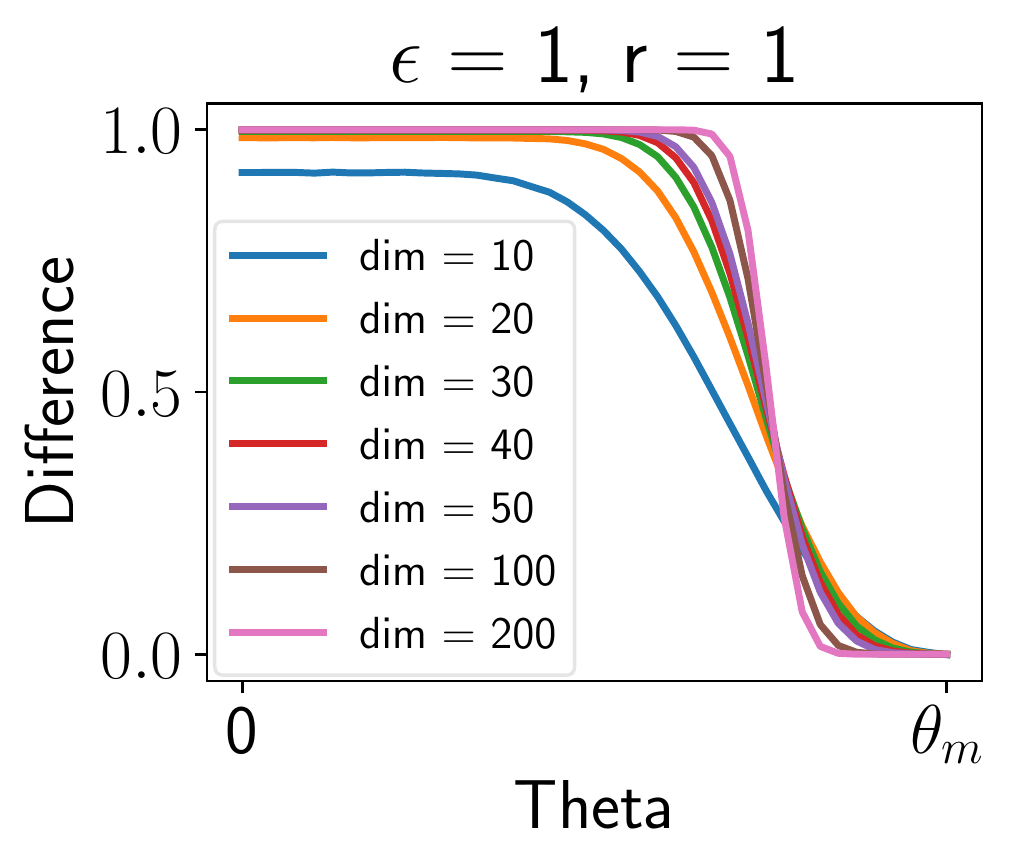}
    \caption{}
    \label{figure:underestimation}
\end{wrapfigure}

As the uniform distribution has a finite support, the corresponding single-noise certificates will be blind to everything outside that support.
It will thus always consider the worst-case scenario, namely that every point outside the ball is of the opposite class.
Since it has no information on the precise repartition of points in the ball, the single-noise certificate will also assume that every point ``lost'' after the translation (see the green left crescent zone outside the intersection in~\Cref{figure:underestimate_cohen_certificate_a}) was of class 1.

Hence, the difference between the perfect certificate and the single-noise certificate is the relative area of the blue zone.
The last part of the result shows that as the dimension gets higher, the single noise certificate can become arbitrarily bad.
In the extreme case ($\nu(\theta) = 1$), it returns 0 even though the classification task is trivial.
This is due to the fact that the volume of balls tends to concentrate on their surface in high dimensions, and so the relative weight of the crescent zone increases.
We used simulations to plot the evolution of the underestimation with the parameter $\theta$ and the dimension. We used Monte Carlo sampling with $200k$ samples for each $\theta$ and each dimension.

As we can see in~\Cref{figure:underestimation}, as dimension increases the cone represents a smaller portion of the space for a given $\theta$. It follows that the underestimation becomes large for almost every $\theta$ in high dimensions.
However, how likely are we to encounter these kinds of high local curvature situations for models trained on real datasets?

\subsection{Empirical analysis with real-world decision boundaries}

In the following, we show that we can identify the points where the single-noise classifier is suboptimal, by leveraging the information from several concentric noises.
More precisely, for the uniform distribution, at points where the single-noise certificate is optimal, the probability of being in class 1 will decrease as the radius of the noise increases.

{\em \textbf{Intuition of~\Cref{proposition:findingUnderestimate}.} 
For an unknown decision boundary, at any point where the single-noise certificate is optimal, the probability must locally decrease with the radius of the noise. It follows that at any point where the probability increases with the radius of the noise, the single-noise certificate cannot be optimal.}

\begin{restatable}[\textbf{Identifying points of non-zero underestimation}]{proposition}{findingUnderestimatePoints}
\label{proposition:findingUnderestimate}
For any $r>0$, let $q_r$ denote the uniform distribution over $B_2^d(0, r)$.
Let $r_1 > 0$, and $h$ a classifier. For any $x\in \mathbb{R}^d$, $\epsilon>0$, if $p(x, h, q_r)$ is not a decreasing function of $r$ over $\left[ r_1, r_1+\epsilon \right)$, then $\PC(h, q_{r_1}, x, \epsilon) > \NC(h, q_{r_1}, x, \epsilon, \{q_r\})$.
\end{restatable}

\vspace{0.2cm}

\begin{wraptable}{r}{0.33\textwidth}
  \centering
  \vspace{-0.5cm}
  \caption{}
  \label{table:proportion_right_curvature_cifar10}
  \begin{tabular}{cc}
  \toprule
  $r_0$ / $r_1$ & $\zeta(r_0, r_1)$ \\
  \midrule
  0.25 / 0.30 & 40.9\% \\
  0.50 / 0.55 & 41.9\% \\
  0.75 / 0.80 & 41.3\% \\
  \bottomrule
  \end{tabular}%
\end{wraptable}

\paragraph{Suboptimality in current models.}
We leverage \Cref{proposition:findingUnderestimate} to evaluate the number of points where single-noise certificates might be optimal, by ``probing'' the decision boundary with different distributions. Let $q_0$ and $q_1$ be two uniform distributions with $r_0$ and $r_1$ as their respective radius such that $r_0 < r_1$ and let $\Dcal$ the set containing all the images of the CIFAR10 dataset~\cite{cifar10}. We aim at computing the proportion $\zeta$ of points where the probability increases with the radius of the noise:
\begin{equation*}
    \zeta(r_0, r_1) = \frac{1}{\left| \Dcal \right|} \sum_{x \in \Dcal} \mathbb{1} \left\{ p(x, h, q_0) < p(x, h, q_1) \right\}
\end{equation*}
This quantity measures the proportion of points where a better information-gathering scheme could improve the certificate.
In the context of randomized smoothing, it is common practice to also add noise during training in order to avoid distribution shift at test time. Therefore, to perform this experiment, we use three models trained by~\citet{yang2020randomized} with uniform distribution with respective radius of $0.25$, $0.50$ and $0.75$. We use the same radius $r_0$ as the one used during training and we use a $r_1 = r_0 + 0.05$.  
Table~\ref{table:proportion_right_curvature_cifar10} shows the results of this experiment.
We observe that nearly half of the corrected classified points have the probability growing suggesting that single-noise certificates would underestimate the real bound.

\vspace{0.2cm}

\section{Bypassing the limitations with noise-based information gathering}
\label{section:neyman_pearson}

We now know that single-noise certificates can heavily underestimate the perfect one for points where the local curvature is high.
This raises the question: is it possible for any noise-based certificate to bypass this issue?
We answer positively to this question by showing that gathering information from more noise distributions at certification time allows to get as close as desired to the perfect certificate, at the expense of high computational cost in the general case.
We then show that with more prior information on the decision boundary, the number of noises required and thus the computational cost can be drastically reduced.

\subsection{A framework for obtaining noise-based certificates}

\begin{restatable}[{\bf Generalized Neyman-Pearson lemma} \citet{chernoff1952generalization}]{lemma}{generalizedneymanpearson}
\label{lemma:generalized_neyman_pearson}
Let $q_0, \dots, q_n$ be probability density functions. For any $k_1,\dots,k_n >0$, we define the Neyman-Pearson set $\mathcal{S}_{\mathcal{K}} = \left\{q_0(x) \leq \sum_{i=1}^{n} k_i q_i(x) \right\}$ and the associated Neyman-Pearson function:
\begin{equation*}
  \Phi_{\mathcal{K}} = \mathbb{1}\left\{ \mathcal{S}_{\mathcal{K}}\right\}
\end{equation*}
Then for any function $\Phi : \mathcal{X} \rightarrow \left[ 0,1 \right]$ such that $\int \Phi q_i \diff \mu \geq \int \Phi_{\mathcal{K}} q_i \diff \mu$ for all $i \in \{1, \dots, n\}$, we have:
\begin{equation*}
\int \Phi_{\mathcal{K}} q_0 \diff \mu \leq \int \Phi q_0 \diff \mu
\end{equation*}
\end{restatable}

The generalized Neyman-Pearson lemma can be reformulated with isotropic noises, to obtain what we call {\em noise-based certificates}.
\begin{corollary}[\textbf{Noise-based certificates}]
\label{corollary:NPcertificate}
Let $\mathcal{Q} = \left\{q_0,\dots,q_n\right\}$ be a finite family of isotropic probability density functions, of same center. Let $\epsilon>0$, and any $\delta$ of norm $\epsilon$. If the $k_i$ are such that $\forall i, p(x,\Phi_{\mathcal{K}},q_i) \leq p(x,h,q_i)$, then we have:
\begin{equation*}
    \NC(h, q_0, x, \epsilon, \mathcal{Q}) \geq  p(x+\delta, \Phi_{\mathcal{K}}, q_0)
\end{equation*}
Furthermore, this becomes an equality if $\forall i, p(x,\Phi_{\mathcal{K}},q_i) = p(x,h,q_i)$
\end{corollary}

This means that by choosing the $k_i$ such that $p(x,\Phi_{\mathcal{K}},q_i)$ is as close as possible from $p(x,h,q_i)$ while remaining lower, we can get arbitrary close to $\NC(h, q_0, x, \epsilon, \mathcal{Q})$ using the Neyman-Pearson classifier $\Phi_{\mathcal{K}}$.

Note an important difference between multiple-noise certificates and single-noise ones: we use many noise distributions of varying amplitude \textit{at certification time}, but the noise used for smoothing at test time, $q_0$, remains constant. That is the strength of our framework: dissociating the information gathering process and the smoothing itself.

An advantage of this method is that since the function class $\mathcal{G}$ can only shrink with the number of noises used, a bound obtained with several noises will always be at least as good as the one proposed by~\citet{cohen2019certified}.

\subsection{Approximation results}

\begin{figure}
    \centering
    \begin{subfigure}[b]{0.48\textwidth}
        \centering
        \scalebox{0.40}{%
          \def\centerarc[#1](#2)(#3:#4:#5)%
    { \draw[#1] ($(#2)+({#5*cos(#3)},{#5*sin(#3)})$) arc (#3:#4:#5); }
\tikzset{%
  >={Latex[width=0mm,length=0mm]},
  base/.style = {fill=white, font=\sffamily},
  dot/.style = {circle, fill=black, minimum size=#1, inner sep=0pt, outer sep=0pt},
  cross/.style = {cross out, draw=black, minimum size=4pt, inner sep=0pt, outer sep=0pt}
}
\begin{tikzpicture}[every node/.style=base, align=center, scale=5]

  \clip (0., -0.8) rectangle (3.0, 0.75);

  \draw[color=myred, line width=2pt] plot [smooth, tension=1] coordinates { (0.0, 0.8) (1.8,-0.6) (3.8,0.8)};
  \draw[fill=myblue, draw=none] ( 4pt,  16pt) rectangle (80+4pt, 16pt+4pt);
  \draw[fill=myblue, draw=none] ( 8pt,  12pt) rectangle (80+4pt, 12pt+4pt);
  \draw[fill=myblue, draw=none] (12pt,   8pt) rectangle (80+4pt,  8pt+4pt);
  \draw[fill=myblue, draw=none] (16pt,   4pt) rectangle (80+4pt,  4pt+4pt);
  \draw[fill=myblue, draw=none] (20pt,   0pt) rectangle (80+4pt,  0pt+4pt);
  \draw[fill=myblue, draw=none] (24pt,  -4pt) rectangle (76+4pt, -4+4pt);
  \draw[fill=myblue, draw=none] (28pt,  -8pt) rectangle (72+4pt, -8+4pt);
  \draw[fill=myblue, draw=none] (36pt, -12pt) rectangle (64+4pt, -12+4pt);
  \draw[fill=myblue, draw=none] (44pt, -16pt) rectangle (56+4pt, -16+4pt);

  \draw[step=4pt,color=darkgray] (0.0,-0.8) grid (3.0,+0.75);

\end{tikzpicture}%
        }%
        \caption{Illustration of Theorem~\ref{theorem:approx_theorem}. By querying the classifier with uniform noises on the squares of the grid, we can compute an approximation of the perfect certificate. As we refine the grid with smaller squares, the approximation becomes increasingly good, and converges to the perfect certificate.}
        \label{figure:decision_boundary}
    \end{subfigure}
    \hfill
    \begin{subfigure}[b]{0.48\textwidth}
        \centering
        \scalebox{0.38}{%
          \def\centerarc[#1](#2)(#3:#4:#5)%
    { \draw[#1] ($(#2)+({#5*cos(#3)},{#5*sin(#3)})$) arc (#3:#4:#5); }
\tikzset{%
  >={Latex[width=0mm,length=0mm]},
  base/.style = {fill=white, font=\sffamily},
  dot/.style = {circle, fill=black, minimum size=#1, inner sep=0pt, outer sep=0pt},
  cross/.style = {cross out, draw=black, minimum size=4pt, inner sep=0pt, outer sep=0pt}
}
\begin{tikzpicture}[every node/.style=base, align=center, scale=5]

  \clip (-0.4, -0.3) rectangle (1.3, 1.3);

  \begin{scope}
    \clip (0.5, 0.5) -- (0.5+0.8, 0.5+0.8) -- (0.5+0.8, 0.5-0.8); 
    \draw[fill=myblue, even odd rule, draw=none]
     (0.4, 0.5) circle (0.3) ;
  \end{scope}

  \begin{scope}
    \clip (0.5, 0.5) -- (0.5+0.8, 0.5+0.8) -- (0.5+0.8, 0.5-0.8); 
    \draw[fill=mygreen, even odd rule, draw=none]
     (0.4, 0.5) circle (0.3) (0.4, 0.5) circle (0.5) ;
  \end{scope}

  \begin{scope}
    \clip (0.5, 0.5) -- (0.5+0.8, 0.5+0.8) -- (0.5+0.8, 0.5-0.8); 
    \draw[fill=myyellow, even odd rule, draw=none]
     (0.4, 0.5) circle (0.5) (0.4, 0.5) circle (0.7) ;
  \end{scope}
  
    \begin{scope}
    \clip (0.5, 0.5) -- (0.5+0.8, 0.5+0.8) -- (0.5, 0.5+0.8); 
    \draw[fill=gray1, even odd rule, draw=none]
     (0.4, 0.5) circle (0.5) (0.4, 0.5) circle (0.7) ;
    \end{scope}
    
    \begin{scope}
    \clip (0.5, 0.5) -- (0.5+0.8, 0.5-0.8) -- (0.5, 0.5-0.8); 
    \draw[fill=gray1, even odd rule, draw=none]
     (0.4, 0.5) circle (0.5) (0.4, 0.5) circle (0.7) ;
    \end{scope}
    
    \begin{scope}
    \clip (0.5, 0.5) -- (0.5+0.8, 0.5+0.8) -- (0.5, 0.5+0.8); 
    \draw[fill=gray2, even odd rule, draw=none]
     (0.4, 0.5) circle (0.3) (0.4, 0.5) circle (0.5) ;
    \end{scope}
    
    \begin{scope}
    \clip (0.5, 0.5) -- (0.5+0.8, 0.5-0.8) -- (0.5, 0.5-0.8); 
    \draw[fill=gray2, even odd rule, draw=none]
     (0.4, 0.5) circle (0.3) (0.4, 0.5) circle (0.5) ;
    \end{scope}
    
    \begin{scope}
    \clip (0.5, 0.5) -- (0.5+0.8, 0.5+0.8) -- (0.5, 0.5+0.8); 
    \draw[fill=gray3, even odd rule, draw=none]
     (0.4, 0.5) circle (0.1) (0.4, 0.5) circle (0.3) ;
    \end{scope}
    
    \begin{scope}
    \clip (0.5, 0.5) -- (0.5+0.8, 0.5-0.8) -- (0.5, 0.5-0.8); 
    \draw[fill=gray3, even odd rule, draw=none]
     (0.4, 0.5) circle (0.1) (0.4, 0.5) circle (0.3) ;
    \end{scope}

  \draw[color=myred, line width=2pt] (0.4, 0.5) circle (0.3);
  \draw[color=myred, line width=2pt] (0.4, 0.5) circle (0.5);
  \draw[color=myred, line width=2pt] (0.4, 0.5) circle (0.7);

  \draw[color=black, line width=2pt] (0.5, 0.5) -- (0.5+0.8, 0.5+0.8);
  \draw[color=black, line width=2pt] (0.5, 0.5) -- (0.5+0.8, 0.5-0.8);
  
  \draw[color=cyan, line width=2pt] (0.5, 1.5) -- (0.5, -1.5);

  \draw[color=black, line width=2pt] (-0.4, 0.5) -- (1.2, 0.5);
  \draw (0.4,0.5) node[cross,rotate=0] {};
  \node[draw=none, fill=white, text opacity=1, fill opacity=0., scale=1.5] at (0.8, 0.57) {$\theta$};
  \node[draw=none, fill=white, text opacity=1, fill opacity=0., scale=1.3] at (0.4, 0.6) {O};

  \centerarc[black](0.4,0.5)(0:33:0.35);

\end{tikzpicture}%
        }%
        \caption{Illustration of \Cref{proposition:successive_noises}. We see that the difference of volume captured between two balls (blue, green and yellow zones) grows with $\theta$. For $\theta < \frac{\pi}{2}$, the volume growth is lower than for a hyperplane decision boundary (the cyan line) at the same distance. The difference is shown by the gray zones.}
        \label{figure:gaussian_sonar}
    \end{subfigure}
    \caption{}
\end{figure}

In this part, we will show that noise-based information is enough to approximate any non-pathological classifier:

{\em \textbf{Intuition of \Cref{theorem:approx_theorem}.} For any continuous decision boundary, it is possible to approximate the perfect certificate with arbitrary precision, using information from a finite family of noise distributions. The size of the family used increases with the desired precision. An illustration of this theorem is shown in \Cref{figure:decision_boundary}.
}

\begin{restatable}[{\bf General approximation theorem}]{theorem}{approximationtheorem}
\label{theorem:approx_theorem}
Let $q_0$ be the uniform noise on the $\ell_\infty$ ball $B_\infty (0 ,r)$ for some $r>0$.
For any $\epsilon>0$, $\xi > 0$ and $x \in \mathbb{R}^d$, there exists a finite family $\mathcal{Q}$ of probability density functions such that:
\begin{equation*}
  \NC(h, q_0, x, \epsilon, \mathcal{Q}) \geq \PC(h, q_0, x, \epsilon) - \xi
\end{equation*}
\end{restatable}

\paragraph{Sketch of proof of~\Cref{theorem:approx_theorem}.}
The idea of this proof is to define a grid of disjoint squares covering the space. Then, we can construct a noise-based classifier that returns 1 only on the squares that are entirely contained in the decision region, \ie, a strict underestimate of the true classifier.
As the grid gets thinner, the approximation will then converge to the true classifier as a Riemann sum. $\square$

\Cref{theorem:approx_theorem} shows that it is possible to collect asymptotically perfect information on the decision boundary using only noise-based queries. The main improvement compared to the result of \citet{mohapatra2020higherorder} is that we reconstruct the base classifier itself, and not just the gaussian smoothed version of it, hence it works for any smoothing scheme.
This shows that the black-box approach to randomized smoothing certification is viable, and can bypass the theoretical limitations when using several noises instead of one.

Furthermore, if we have access to some prior information on the decision boundary, it will be possible to design much more efficient noise-based information gathering schemes.
In the following, we present a result demonstrating that we can obtain full information in the case of conical or 2-piecewise linear decision boundaries by using only a few concentric uniform noises.

\begin{definition}[\textbf{Noised-based certificate with prior information}]
\label{definition:noised_based_certificate_with_prior}
Let $\mathcal{Q}$ be a finite family of probability density functions, $\mathcal{F}$ be a family of classifiers (typically parameterized).
Let $q_0 \in \mathcal{Q}$. The $\mathcal{Q}-$noise-based $\epsilon$-certificate with prior information $\mathcal{F}$ for the $q_0$-randomized smoothing of $h$ at point $x$ is: 
\begin{equation*}
  \NCP(h, q_0, x,\epsilon, \mathcal{Q}, \mathcal{F}) = \inf\limits_{g \in \mathcal{G}_{\mathcal{Q},\mathcal{F}} } \inf\limits_{\delta \in B(0,\epsilon)} p(x+\delta, g, q_0) 
\end{equation*}
where:
\begin{equation*}
\mathcal{G}_{\mathcal{Q}, \mathcal{F}}=\left\{ g \in \mathcal{F} \mid \forall q \in \mathcal{Q},\ p(x,g,q) = p(x,h,q) \right\}
\end{equation*}
\end{definition}

{\em \textbf{Intuition of~\Cref{proposition:successive_noises}.} If we know that the decision boundary is conical or 2-piecewise-linear, then its parameter $\theta$ can be perfectly identified using only the information from two concentric noises. This allows the perfect certificate to be computed as a noise-based certificate.
\Cref{figure:gaussian_sonar} is an intuitive illustration of the proposition.
We can see that the volume of the cone captured by the balls is a strictly non-decreasing function of $\theta$. }

\begin{restatable}[{\bf Perfect certificate for conical decision boundaries}]{proposition}{successivenoises}
\label{proposition:successive_noises}
Let $\theta_0 \in \left[ 0,\frac{\pi}{2} \right]$. Let $h$ be a classifier whose decision boundary is the cone $C(0,\theta_0)$. Let $\mathcal{F}$ be the family of all classifiers with a decision boundary of the form $C(0,\theta)$. Then there exists uniform noises $q_1$ and $q_2$ such that, for any noise $q_0$, and any $x\in \mathbb{R}^d, \epsilon>0$:
\begin{equation*}
    \NCP(h,q_0,x,\epsilon,\left\{q_1,q_2\right\}, \mathcal{F}) = \PC(h,q_0,x,\epsilon)
\end{equation*}
\end{restatable}

Also, by a direct extension of \Cref{proposition:successive_noises}, a small number of concentric noises are enough to obtain full information on general cones $C(c,\theta)$, in two steps:
\begin{itemize}[topsep=0pt,parsep=0pt,leftmargin=0.3cm]
    \item Evaluate the distance $c$ to the decision boundary by finding the threshold such that $p(x, h, q(r)) \neq 1$;
    \item Use two noises of radius $r_1$ and $r_2$ to identify the angle $\theta$ as presented in~\Cref{proposition:successive_noises}.
\end{itemize}
This hints at a more general result for piecewise linear decision boundaries (which includes all neural networks with ReLUs activations): it may be possible to gather perfect information using only a limited number of concentric noises to ``map'' the fractures of the decision boundary.

Designing certificates thus shifts from a classifier-agnostic problem to a more classifier-specific one: any prior information on the decision boundary can help guide the choices of noises used at certification time.
This also suggests that we could also choose the base classifier not only because of its efficiency, but to obtain some desirable properties that facilitate the certification process. This opens up a wide area of research.

\section{Discussion on computational cost}
\label{section:experiments}

In this section, we analyze the computational challenges of implementing noise-based certificates, and explore some avenues to reduce them. There are currently three main obstacles to computing noise-based certificates using the generalized Neyman-Pearson Lemma:

\begin{enumerate}[parsep=0pt,itemsep=0pt,topsep=0pt,leftmargin=0.5cm]
\item Computing integrals via Monte Carlo sampling in high-dimension can become very costly as this technique suffers from the curse of dimensionality.
\item When computing the integrals in high dimensions, numbers can become very small or very large, leading to computational instability.\footnote{For example, the volume of an $\ell_2$ ball in dimension 784 (MNIST dimension) is approximately equal to $\exp(-1503.90)$.}
\item Finally, fitting the $k_i$ to compute the generalized Neyman-Pearson set is a hard stochastic optimization problem.
\end{enumerate}

We show that we can bypass problems $1.$ and $2.$ when using Gaussian noise for information collection. Furthermore, uniform noise as an information gathering method considerably reduces problem $3$, although suffering from problems $1$ and $2$.

\begin{restatable}[{\em Computing the $k_i$ for uniform noises}]{proposition}{uniformKiCombinat}
\label{prop:uniformKiCombinat}
Let $q_0,\dots,q_n$ are uniform distributions, where $n \ll d$ there are only at most $2^n$ possible values for the generalized Neyman-Pearson set $\Scal$.
\end{restatable}

This means that the exact values of the $k_i$ do not matter, only the possible values of the Neyman-Pearson set $\Scal$. The research of $\Scal$ thus shifts from a hard optimization problem to a combinatorial problem with only at most $2^n$ values to try where $n$ correspond to the number of noise and is usually much lower than the input dimension. Note that by taking $n = 1$, the certificate reduces to the single-noise certificate, and increasing the number of noises can only improve it.
Also, we should remark that smart choices of noises can make that combinatorial problem easier in practice, since the support of the distributions does not necessarily intersect with each other.

\begin{restatable}[{\bf Gathering information from gaussian noises}]{theorem}{differensigma} 
\label{theorem:different_sigma}
Let $q_0$ be any isotropic probability distribution, $\sigma_1, \dots, \sigma_n > 0$. For $i=1\dots n$ let $q_i \sim \mathcal{N}(0,\sigma_i)$ be the noises used for information gathering. Let $S_{k_1\dots k_n}$ be the corresponding Neyman-Pearson set, for any combination of parameters $k_1,\dots, k_n >0$.

Then $\mathbb{P}[\mathcal{N}(0,\sigma^2_i) \in S_{k_1\dots k_n}]$ can be computed using a Monte Carlo sampling in dimension 2 from a $\chi$ distribution with $d-2$ degrees of freedom.
\end{restatable}

\textbf{Sketch of proof for \Cref{theorem:different_sigma}.} The key of this proof is again the invariance by rotation of the generalized Neyman-Pearson set around the direction $e_1$ of the attack. This allows us to separate $\|z \|^2$ into two components, one along $e_1$, which follows a 1-dimensional normal distribution, and one in $e_1^{\perp}$, whose norm follows a $\chi$ distribution with $d-2$ degrees of freedom. $\square$

This means that whatever the noise used at smoothing time, we can easily gather information from gaussian noises, since the Neyman-Pearson set needs only be sampled in dimension 2 to fit the $k_i$.

\section{Conclusion \& Future Works}
\label{section:conclusion}

\textbf{Computing the $k_i$.}
\Cref{theorem:different_sigma} successfully reduces the difficulty of the problem.
However, even with those simplifications, fitting the $k_i$ of the generalized Neyman-Pearson set remains a difficult stochastic optimization problem. Indeed, each step requires the computation of an integral via Monte Carlo sampling, and many steps may be necessary to reach the desired precision.
A potential direction of research would be to use the relaxation introduced by~\citet{dvijotham2020framework} for an easier to compute approximation of the Neyman-Pearson set.
Both techniques from~\citet{yang2020randomized} to compute ordinary Neyman-Pearson sets can also be extended to our general sets, for more computational efficiency.

\textbf{Choosing the base classifier $h$.}
Now that our certificates use more specific information on the classifier, it is possible to optimize the combination between the base classifier and the noise distributions used.
For example, we may adjust our training to ensure that the decision boundary has the highest possible curvature, since it is where our new certificates will shine.
The work from~\citet{salman2019provably}, which combined noise injection and adversarial training~\cite{madry2018towards} during the training, suggest that different training schemes can have an important impact on the certification performance.
Recently, this line of research has been studied and further improvements have been devised~\cite{zhai2020macer,jeong2020consistency,zhen2021simpler,wang2021pretraintofinetune}.
In the context of our framework, new training schemes could be devised to improve the local curvature at each point by adjusting the amount of noise injected.

\textbf{Conclusion.}
We have shown that the limitations of randomized smoothing are a byproduct of the certification method, namely the combination of the smoothing and information gathering steps.
We show that by dissociating the two processes, and using multiple distributions for the information gathering, it is possible to circumvent these limitations without affecting the standard accuracy of the classifier.
This opens up a whole new field of classifier-specific certification, with the guarantee of always performing better than single-noise certificates, and without any additional loss in standard accuracy.
Furthermore, it is now possible to optimize the choice of the base classifier, and use prior information in the certification process.
Much work remains to be done, in order to actually implement certificates using this framework.
The main difficulty is to compute the worst-case decision boundary, by fitting the constants $k_i$. 
This is hard to do using Monte Carlo sampling for most choices of noises. 
But the obstacles have now shifted from an impossibility result to computational challenges, restoring hope that randomized smoothing may someday be a definitive solution against adversarial attacks.

\newpage
\bibliographystyle{plainnat}
\bibliography{bibliography}

\clearpage
\appendix

\onecolumn

\section{Definitions}

In the following, we will consider the dimension $d \geq 3$.

\begin{definition}[\textbf{Hyper-cylindrical coordinates} This is an extension of the hyperspherical coordinates, defined in \citet{blumenson1960derivation}]
\label{definition:hyper_cylindrical_coordinates}
Let $e_1,\dots,e_d$ be an orthonormal base of $\mathbb{R}^d$, with corresponding Euclidean coordinates $(x_1,\dots,x_d)$. The hyper-cylindrical coordinates of axis $e_1$ are the following change of variable:
\begin{align}
    z &= x_1 \\
    \rho &= \sqrt{x^2_2 + \dots + x^2_d} \\
    \phi_i &= \arccot \left( \frac{x_i}{\sqrt{x_n^2+\dots+x_i^2}} \right) \\
    \phi_{d-1} &= 2 \arccot\left( \frac{x_{d-1} + \sqrt{x_{d-1}^2 + x_d^2}}{x_d} \right)
\end{align}
with the following reverse transformation:
\begin{align}
    x_1 &= z \\
    x_2 &= r \cos(\phi_1) \\
    x_i &= r \left( \prod_{i=1}^{i-2} \sin(\phi_i) \right) \cos(\phi_{i-1}) \\
    x_d &= r \left( \prod_{i=1}^{d-2} \sin(\phi_{i-2}) \right)
\end{align}
where $i\in \left\{ 2,\dots,d-1\right\}$.
This is a bijection, where $\phi_i \in \left[0,\pi\right]$, $r \in \mathbb{R}_+$, and $\phi_{d-1} \in \left[0,2\pi\right]$, with the convention that $\phi_k=0$ when $x_k,\dots,x_n = 0$.
Note that it is simply a change of variables to hyperspherical coordinates on the $d-1$ last variables.
\end{definition}

\begin{definition}[\textbf{Incomplete Regularized Beta}]
\label{definition:incomplete_beta_function}
Let $z \in \Rbb$, $a > 0$ and $b > 0$.
The Incomplete Regularized Beta Function is the function defined as:
\begin{equation} 
  I_z(a, b) = \frac{\Gamma(a+b)}{\Gamma(a)\Gamma(b)} \int_0^z t^{a-1} (1-t)^{b-1} \diff t
\end{equation}
\end{definition}

\begin{definition}[\textbf{$\ell_p$-ball}]
\label{definition:lp_ball}
The $\ell_p$-ball of dimension $d$, radius $r > 0$ and center $c \in \mathbb{R}^d$ is the set:
\begin{equation}
B_p^d(c, r) = \{ \forall x \in \mathbb{R}^d \mid \norm{x}_p \leq r \}
\end{equation}
\end{definition}

\begin{definition}[\textbf{Spherical Cap of an $\ell_p$-ball} \citep{li2011concise}]
\label{definition:spherical_cap}
A {\em spherical cap} is the portion of the sphere that is cut away by an hyperplane of distance $r - a$ from the origin.
The formula for the volume of the spherical cap is given by:
\begin{equation}
  \Vol(\SphereCap(a, r, d)) = \frac{1}{2} \Vol(B_p^d(0, r)) I_{\frac{2ra - h^2}{r^2}} \left( \frac{d+1}{2}, \frac{1}{2} \right)
\end{equation}
\end{definition}

\begin{definition}[\textbf{Cone of revolution}]
Let $c\geq0$. For any $x \in \mathbb{R}^d$, let $z, \rho, \phi_1,\dots, \phi_{d-2}$ be the hyper-cylindrical coordinates of axis $e_1$ (see~\Cref{definition:hyper_cylindrical_coordinates}). The cone of revolution of axis $e_1$, peaked at $c$ and of angle $\theta \in \left[ 0, \frac{\pi}{2} \right]$ is the set $\mathcal{C}(c,\theta)$, defined by:
\begin{equation*}
    \left\{
    \begin{array}{l|l}
        z \in \mathbb{R} \\
        \rho \in \mathbb{R}_{+} & z > c \text{\em \ and\ } \rho \leq z \tan \theta \\
        \phi_1,\dots,\phi_{d-2} \in \left[ 0,\pi \right].
    \end{array}
    \right\}
\end{equation*}
when $\theta \leq \frac{\pi}{2}$ (convex cone), and the set: 
\begin{equation*}
\mathcal{C}(c,\theta) = \left\{ z \geq c \text{\em \ or\ } \rho \geq -z \tan(\pi-\theta) \right\}.
\end{equation*}
for the concave cone ($\theta > \frac{\pi}{2}$).

We define a classifier with conical decision boundary as $h_\theta: x \mapsto \mathbb{1}\left\{ x \notin \mathcal{C}(c, \theta)\right\}$.
\end{definition}

\begin{definition}[{\bf 2-piecewise Linear set}]
Let $c\geq0$. Let $x_1,\dots x_n$ be the euclidean coordinates in the base $(e_1,\dots, e_n)$. The $2$-piecewise linear decision region of axis $e_1$ and $e_2$, of distance $c$ and angle $\theta \in \left[ 0, \frac{\pi}{2}\right]$ is the set:
\begin{equation*}
  \left\{ x_1,\dots,x_n \in \mathbb{R} \mid x_1 > c \text{\em \ and } \arctan(\frac{x_2}{x_1}) \in \left[ -\theta, \theta \right]\right\}
\end{equation*}
\end{definition}

\begin{definition}[{\bf Linear half-space}]
Let $c\geq0$. The half-space of translation $c$ is, in hypercylindrical coordinates of axis $e_1$, the set:
\begin{equation*}
H(c) = \left\{ \begin{array}{l|l}
        z \in \mathbb{R} \\
        \rho \in \mathbb{R}_{+} & z > c \\
        \phi_2,\dots,\phi_{d-1} \in \left[ 0,\pi \right].
    \end{array} \right\}
\end{equation*}
\end{definition}

\section{Proofs of \Cref{section:limitations_rs}}

\subsection{Proof of \Cref{theorem:underestimate_cohen}} 
\underestimateprop*

\begin{lemma}[\textbf{Limit of the regularized incomplete beta function}]
\label{lemma:limitBetaFunction}
Let $z \leq 1$, $b = \frac{1}{2}$ fixed.
Then $I_z(a,b) \xrightarrow[a \to \infty]{} 0$. 
\end{lemma}

\begin{proof}
For any $z < 1, a > 1, b = \frac{1}{2}$, we have:
\begin{align}
    I_z(a, b) &= \frac{\Gamma(a+b)}{\Gamma(a)\Gamma(b)} \int_0^z t^{a-1} (1-t)^{b-1} \diff t \\
    &\leq \frac{\Gamma(a+b)}{\Gamma(a)\Gamma(b)} z^a (1-z)^{-\frac{1}{2}}
\end{align}
From \citet{olver2010nist}, Equation~$5.11.12$, we have the following approximation:
\begin{equation} \label{equation:ratio_gamma}
    \frac{\Gamma(a+b)}{\Gamma(a)} \underset{a \to +\infty}{\sim} a^{b}
\end{equation}
from \Cref{equation:ratio_gamma}, we can show that:
\begin{align}
    \frac{\Gamma(a+b)}{\Gamma(a)\Gamma(b)} &\underset{a \to +\infty}{\sim} \frac{a^b}{\Gamma(b)}
\end{align}
Finally, we have:
\begin{align}
    I_z(a,b) &\leq \frac{\Gamma(a+b)}{\Gamma(a)\Gamma(b)} z^a (1-z)^{-\frac12} \\
    &\underset{a \to +\infty}{\sim} \frac{a^{b}}{\Gamma(b)} z^{a} (1-z)^{-\frac12} \\
    &\xrightarrow[a \to +\infty]{} 0
\end{align}
which concludes the proof.
\end{proof}

In what follows, $V = \Vol(B_2^d(0,r))$

\begin{lemma} \label{lemma:optimal_attack_e1}
The optimal attack of size $\epsilon < r$ against $\mathcal{C}(0,\theta)$ is the translation fully along its axis $e_1$, \ie $\epsilon e_1$.
\end{lemma}
\begin{proof}
Let $A = \Span(e_1), B = \Span(e_2,\dots, e_d)$. For any vector $u \in \mathbb{R}^d$, we write $u = u_A + u_B$ where $u_A$ and $u_B$ are the orthogonal projections of $u$ on $A$ and $B$ respectively.

First we will show that since the cone is invariant by rotation around $e_1$, the orthogonal component of the attack is as well. Without any attack, the probability of returning 1 at point $0$, is
\begin{equation}
\frac{1}{V} \int\limits_{\mathbb{R}^d}\mathbb{1}\left\{ \| x_A - 0 \|^2 + \| x_B - 0 \|^2 \leq r^2 \right\} \mathbb{1}\left\{ \| x_B\| \leq \| x_A\| \tan(\theta) \right\} \diff x
\end{equation}
where $V$ is the volume of the ball of radius $r$ and center $0$.

Let $\delta$ be any attack vector, $\| \delta \| = \epsilon$.
Attacking by $\delta$ amounts to shifting the center of the ball from $(0,0)$ to $(\delta_A,\delta_B)$.
Let
\begin{equation}
f(x,\delta) = \mathbb{1}\left\{ \| x_A - \delta_A\|^2 + \| x_B - \delta_B\|^2 \leq r^2 \right\} \mathbb{1}\left\{ \| x_B\| \leq \| x_A\| \tan(\theta) \right\}, \quad \quad \forall x \in \mathbb{R}^d
\end{equation}
Then the probability of returning 1 at point $0$, under attack $\delta$, is $p(\delta) = \frac{1}{V}\int\limits_{\mathbb{R}^d}f(x,\delta)\diff x$ where $V$ is independent of $\delta$.

Now let $g$ be any isometric mapping such that $\left.g\right|_A = \mathrm{Id}_A$. 
Let $\tilde{\delta} = g(\delta)$. Recall that as $g$ is an isometry hence $g$ and $g^{-1}$ are also affine, hence we have
\begin{align}
    f(g^{-1}(x), \delta) &= \mathbb{1}\left\{ \| g^{-1}(x_A) - \delta_A\|^2 + \| g^{-1}(x_B) - \delta_B\|^2 \leq r^2 \right\} \mathbb{1}\left\{ \| g^{-1}(x_B)\| \leq \| g^{-1}(x_A)\| \tan(\theta) \right\} \\
    &= \mathbb{1}\left\{ \| g^{-1}(x_A - \tilde{\delta}_A)\|^2 + \| g^{-1}(x_B - \tilde{\delta}_B)\|^2 \leq r^2 \right\} \mathbb{1}\left\{ \| g^{-1}(x_B)\| \leq \| g^{-1}(x_A)\| \tan(\theta) \right\} \\
    &= \mathbb{1}\left\{ \| x_A - \tilde{\delta}_A\|^2 + \| x_B - \tilde{\delta}_B\|^2 \leq r^2 \right\} \mathbb{1}\left\{ \| x_B\| \leq \| x_A\| \tan(\theta) \right\} \\
    &= f(x, \tilde{\delta}) = f(x,g^{-1}(\delta))
\end{align}
since $g^{-1}$ is an isometry.
It follows, by a change of variable in the integral:
\begin{align}
    p(\delta) &= \frac{1}{V}\int\limits_{\mathbb{R}^d}f(x,\delta) \diff x \\
    &= \frac{1}{V}\int\limits_{\mathbb{R}^d}f(g^{-1}(u),\delta)\diff u \\
    &= \frac{1}{V}\int\limits_{\mathbb{R}^d}f(u,g^{-1}(\delta))\diff u \\
    &=p(\tilde{\delta})
\end{align}

In particular, we can always choose $g$ such that the $g(\delta_B) = \delta_2 e_2$. In what follows, we will consider $\delta$ of the form $\delta = \delta_1 e_1 + \delta_2 e_2$ and show that the attack is optimal when $\delta_2 = 0$. For this, we will compute the difference between the attack translated by $\delta_1 e_1$ and the one translated by $\delta_1 e_1 + \delta_2 e_2$, to show that the orthogonal component actually reduces the efficiency of the attack. Let us first recall that 
\begin{align}
    p(\delta) &= \frac{1}{V}\int\limits_{\mathbb{R}^d} \mathbb{1}\left\{ \| x_1 - \delta_1\|^2 + \| x_2 - \delta_2\|^2 + \| x_3 \|^2 + \dots + \| x_d \|^2 \leq r^2 \right\} \mathbb{1}\left\{ \| x_B\| \leq \| x_A\| \tan(\theta) \right\}\diff x_1 \dots \diff x_d \\
    &= \frac{1}{V}\Vol(B(\delta,r)\cap C(0, \theta)).
\end{align}

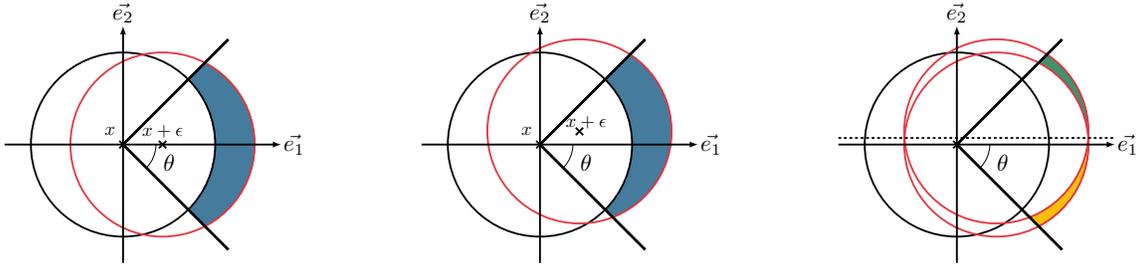
\begin{figure}[th]
  \centering
  \hfill
  \begin{subfigure}[h]{0.32\textwidth}
    \centering
    \scalebox{0.35}{%
      \def\centerarc[#1](#2)(#3:#4:#5)%
    { \draw[#1] ($(#2)+({#5*cos(#3)},{#5*sin(#3)})$) arc (#3:#4:#5); }
\tikzset{%
  >={Latex[width=0mm,length=0mm]},
  base/.style = {fill=white, font=\sffamily},
  dot/.style = {circle, fill=black, minimum size=#1, inner sep=0pt, outer sep=0pt},
  cross/.style = {cross out, draw=black, minimum size=4pt, inner sep=0pt, outer sep=0pt}
}
\begin{tikzpicture}[every node/.style=base, align=center, scale=5]

  \clip (-0.4, -0.4) rectangle (1.9, 1.6);

  \begin{scope}
    \clip (0.5, 0.5) -- (0.5+0.8, 0.5+0.8) -- (0.5+1.2, 0.5+0.8) -- (0.5+1.2, 0.5-0.8) -- (0.5+0.8, 0.5-0.8) -- (0.5, 0.5);
    \draw[fill=myblue, even odd rule] (0.8, 0.5) circle (0.7) (0.5, 0.5) circle (0.7) ;
  \end{scope}

  \draw[color=black, line width=2pt] (0.5, 0.5) circle (0.7);
  \draw[color=myred, line width=2pt] (0.8, 0.5) circle (0.7);

  \draw[color=black, line width=2pt, -latex] (-0.4, +0.5) -- (+1.7, +0.5);
  \draw[color=black, line width=2pt, -latex] (+0.5, -0.4) -- (+0.5, +1.4);

  \draw (0.5, 0.5) node[cross,rotate=0,scale=2,line width=2] {};
  \draw (0.8, 0.5) node[cross,rotate=0,scale=2,line width=2] {};
 
  \node[draw=none, fill=white, text opacity=1, fill opacity=0., scale=2.5] at (1.8, 0.5) {$\vec{e_1}$};
  \node[draw=none, fill=white, text opacity=1, fill opacity=0., scale=2.5] at (0.5, 1.5) {$\vec{e_2}$};
  \node[draw=none, fill=white, text opacity=1, fill opacity=0., scale=2.5] at (0.85,  0.36) {$\theta$};

  \node[draw=none, fill=white, text opacity=1, fill opacity=0., scale=2.0] at (0.40,  0.6) {$x$};
  \node[draw=none, fill=white, text opacity=1, fill opacity=0., scale=2.0] at (0.45+0.35,  0.6) {$x+\epsilon$};

  \draw[color=black, line width=3pt] (0.5, 0.5) -- (0.5+0.8, 0.5+0.8);
  \draw[color=black, line width=3pt] (0.5, 0.5) -- (0.5+0.8, 0.5-0.8);

  \centerarc[black](0.5, 0.5)(0:-45:0.25);

\end{tikzpicture}%
    }%
  \end{subfigure}
  \hfill
  \begin{subfigure}[h]{0.32\textwidth}
    \centering
    \scalebox{0.35}{%
      \def\centerarc[#1](#2)(#3:#4:#5)%
    { \draw[#1] ($(#2)+({#5*cos(#3)},{#5*sin(#3)})$) arc (#3:#4:#5); }
\tikzset{%
  >={Latex[width=0mm,length=0mm]},
  base/.style = {fill=white, font=\sffamily},
  dot/.style = {circle, fill=black, minimum size=#1, inner sep=0pt, outer sep=0pt},
  cross/.style = {cross out, draw=black, minimum size=4pt, inner sep=0pt, outer sep=0pt}
}
\begin{tikzpicture}[every node/.style=base, align=center, scale=5]

  \clip (-0.4, -0.4) rectangle (1.9, 1.6);

  \begin{scope}
    \clip (0.5, 0.5) -- (0.5+0.8, 0.5+0.8) -- (0.5+1.2, 0.5+0.8) -- (0.5+1.2, 0.5-0.8) -- (0.5+0.8, 0.5-0.8) -- (0.5, 0.5);
    \draw[fill=myblue, even odd rule] (0.8, 0.6) circle (0.7) (0.5, 0.5) circle (0.7) ;
  \end{scope}
   
  \draw[color=black, line width=2pt] (0.5, 0.5) circle (0.7);
  \draw[color=myred, line width=2pt] (0.8, 0.6) circle (0.7);

  \draw[color=black, line width=2pt, -latex] (-0.4, +0.5) -- (+1.7, +0.5);
  \draw[color=black, line width=2pt, -latex] (+0.5, -0.4) -- (+0.5, +1.4);

  \draw (0.5, 0.5) node[cross,rotate=0,scale=2,line width=2] {};
  \draw (0.8, 0.6) node[cross,rotate=0,scale=2,line width=2] {};

  \node[draw=none, fill=white, text opacity=1, fill opacity=0., scale=2.5] at (1.8, 0.5) {$\vec{e_1}$};
  \node[draw=none, fill=white, text opacity=1, fill opacity=0., scale=2.5] at (0.5, 1.5) {$\vec{e_2}$};
  \node[draw=none, fill=white, text opacity=1, fill opacity=0., scale=2.5] at (0.85,  0.36) {$\theta$};

  \node[draw=none, fill=white, text opacity=1, fill opacity=0., scale=2.0] at (0.40,  0.6) {$x$};
  \node[draw=none, fill=white, text opacity=1, fill opacity=0., scale=2.0] at (0.45+0.40,  0.67) {$x+\epsilon$};

  \draw[color=black, line width=3pt] (0.5, 0.5) -- (0.5+0.8, 0.5+0.8);
  \draw[color=black, line width=3pt] (0.5, 0.5) -- (0.5+0.8, 0.5-0.8);

  \centerarc[black](0.5, 0.5)(0:-45:0.25);

\end{tikzpicture}%
    }%
    \label{figure:lemma_fig2}
  \end{subfigure}
  \hfill
  \begin{subfigure}[h]{0.32\textwidth}
    \centering
    \scalebox{0.35}{%
      \def\centerarc[#1](#2)(#3:#4:#5)%
    { \draw[#1] ($(#2)+({#5*cos(#3)},{#5*sin(#3)})$) arc (#3:#4:#5); }
\tikzset{%
  >={Latex[width=0mm,length=0mm]},
  base/.style = {fill=white, font=\sffamily},
  dot/.style = {circle, fill=black, minimum size=#1, inner sep=0pt, outer sep=0pt},
  cross/.style = {cross out, draw=black, minimum size=4pt, inner sep=0pt, outer sep=0pt}
}
\begin{tikzpicture}[every node/.style=base, align=center, scale=5]

  \clip (-0.4, -0.4) rectangle (1.9, 1.6);

  \begin{scope}
    \clip (0.5, 0.5) -- (0.5+0.8, 0.5+0.8) -- (0.5+1.2, 0.5+0.8) -- (0.5+1.2, 0.5) -- (0.5, 0.5);
    \draw[fill=mygreen, even odd rule] (0.8, 0.6) circle (0.7) (0.8, 0.5) circle (0.7) ;
  \end{scope}

  \begin{scope}
    \clip (0.5, 0.5) -- (0.5+0.8, 0.5-0.8) -- (0.5+1.2, 0.5-0.8) -- (0.5+1.2, 0.5) -- (0.5, 0.5);
    \draw[fill=myyellow, even odd rule] (0.8, 0.6) circle (0.7) (0.8, 0.5) circle (0.7) ;
  \end{scope}
   
  \draw[color=black, line width=2pt] (0.5, 0.5) circle (0.7);
  \draw[color=myred, line width=2pt] (0.8, 0.6) circle (0.7);
  \draw[color=myred, line width=2pt] (0.8, 0.5) circle (0.7);

  \draw[color=black, line width=2pt, -latex] (-0.4, +0.5) -- (+1.7, +0.5);
  \draw[color=black, line width=2pt, -latex] (+0.5, -0.4) -- (+0.5, +1.4);

  \draw (.5,.5) node[cross,rotate=0,scale=2,line width=2] {};

  \node[draw=none, fill=white, text opacity=1, fill opacity=0., scale=2.5] at (1.8, 0.5) {$\vec{e_1}$};
  \node[draw=none, fill=white, text opacity=1, fill opacity=0., scale=2.5] at (0.5, 1.5) {$\vec{e_2}$};
  \node[draw=none, fill=white, text opacity=1, fill opacity=0., scale=2.5] at (0.85,  0.36) {$\theta$};

  \draw[color=black, line width=2pt, dashed] (-0.4, +0.5+0.05) -- (+1.7, +0.5+0.05);

  \draw[color=black, line width=3pt] (0.5, 0.5) -- (0.5+0.8, 0.5+0.8);
  \draw[color=black, line width=3pt] (0.5, 0.5) -- (0.5+0.8, 0.5-0.8);

  \centerarc[black](0.5, 0.5)(0:-45:0.25);

\end{tikzpicture}%
    }%
    \label{figure:lemma_fig3}
  \end{subfigure}
  \caption{Illustration of the proof. The illustration on the left shows that there is always a gain by translating along $e_1$, the illustration in the middle shows the gain when translating along both $e_1$ and $e_2$, and finally, the illustration on the right shows the difference. The loss incurred by the second translation, visible in the yellow zone, is greater than the gain (green zone). The argument of the proof is that the symmetric of the blue zone is contained in the yellow zone.}
  \label{figure:optimal_attac_proof}
\end{figure}

Let $A = B(\delta,r) \cap C(0, \theta) \setminus B(\delta_1 e_1, r)$ and $D = B(\delta_1 e_1,r) \cap C(0, \theta) \setminus B(\delta, r)$.
\begin{align}
    p(\delta_1 e_1) - p(\delta) &= \frac{1}{V} \Vol(B(\delta,r)\cap C(0, \theta)) - \frac{1}{V}\Vol(B(\delta_1 e_1,r)\cap C(0, \theta)) \\ 
    &= \frac{1}{V} \left( \Vol (B(\delta,r)\cap C(0, \theta)\cap B(\delta_1 e_1,r) ) +  \Vol \left( B(\delta,r)\cap C(0, \theta)\setminus B(\delta_1 e_1,r) \right) \right. \\
    &\quad\quad - \left. \Vol (B(\delta_1 e_1,r)\cap C(0, \theta))\cap B(\delta,r)\right) -  \Vol \left( B(\delta_1 e_1,r)\cap C(0, \theta)\setminus B(\delta,r) \right) \\
    &= \frac{1}{V} \left( \Vol(D)-\Vol(A) \right)
\end{align}

We have $p(\delta_1 e_1) - p(\delta) = \frac{1}{V} \left( \Vol(D)-\Vol(A) \right)$. To show that it is positive, we will show that there is an isometry $v$ (preserving volumes) such that $v(A) \subset D$.

Let $v$ be the reflection across the hyperplan $\left\{ x \in \mathbb{R}^d, x_2 =\frac{\delta_2}{2}\right\}$. We have $v(x_1,\dots,x_d) = (x_1, \delta_2-x_2, x_3, \dots,x_d)$. For simplicity, for any $x \in \mathbb{R}^d$, we denote $v(x)=\tilde{x}$.

Let $x\in A$. We will show that $\tilde{x}$ is in $D$.
As $x$ is in $A$, we have
\begin{empheq}[left=\empheqlbrace]{align}
(x_1-\delta_1)^2 + (x_2-\delta_2)^2 + x_3^2 + \dots + x_d^2 \leq r^2 \label{eq1} \\
x_1 > 0 \label{eq2}\\
x_2^2 + \dots + x_d^2 \leq x_1^2 \tan^2(\theta) \label{eq3}\\
(x_1 -\delta_1)^2 + x_2^2 + x_3^2 + \dots + x_d^2 > r^2 \label{eq4} \end{empheq}

\Cref{eq1} states that $x \in B(\delta,r)$, \Cref{eq2} and \Cref{eq3} state that it is in the cone, whereas \Cref{eq4} says that $x \notin B(\delta_1 e_1, r)$.

Let us first show that $\tilde{x} \in C(0, \theta)$.
$\tilde{x_1} = x_1>0$, and subtracting \Cref{eq1} from \Cref{eq4} gives us $x_2^2 > (x_2-\delta_2)^2$.
It follows:
\begin{align}
    \tilde{x_2}^2 + \dots + \tilde{x_d}^2 &= (\delta_2 - x_2)^2 + x_3^2 + \dots + x_d^2 \\
    &< x_2^2 + \dots + x_d^2 \\
    &\leq x_1^2 \tan^2(\theta) \quad \quad \text{(from \Cref{eq3})} \\
    &= \tilde{x_1} \tan^2(\theta)  
\end{align}

Now we show that $\tilde{x} \in B(\delta_1 e_1,r)$.
\begin{align}
    (\tilde{x_1} -\delta_1)^2 + \tilde{x_2}^2 + \dots + \tilde{x_d}^2
    &= (x_1-\delta_1)^2 + (x_2-\delta_2)^2 + x_3^2 + \dots + x_d^2 \\
    &\leq r^2
\end{align}

Finally we show $\tilde{x} \notin B(\delta,r)$.
\begin{align}
    (\tilde{x_1} -\delta_1)^2 + (\tilde{x_2}-\delta_2)^2 + \dots + \tilde{x_d}^2
    &= (x_1-\delta_1)^2 + x_2^2 + x_3^2 + \dots + x_d^2 \\
    &> r^2 \quad \quad \text{(from \Cref{eq4})}
\end{align}

Combining the above, we get  $\tilde{x} \in D$. As $x$ was chosen arbitrarily in $A$, we get $v(A) \subset D$. As $v$ is isometric, we finally get $p(\delta_1 e_1) - p(\delta) =\frac{1}{V} \left( \Vol(A) - \Vol(D) \right) = \frac{1}{V} \left( \Vol(v(A)) - \Vol(D) \right) \leq 0$. The component orthogonal to the axis is detrimental to the attack.

We now only need to prove that $p(\delta)$ is strictly increasing with $\delta_1$. For what follow, we consider an attack $\delta_1 e_1$, and another one $((\delta_1 + \Delta) e_1)$. We will use the same technique:

Let $A = B(\delta_1 e_1, r) \cap C(0, \theta) \setminus B((\delta_1 + \Delta)e_1,r)$, and $D =  B((\delta_1 + \Delta)e_1,r) \cap C(0, \theta) \setminus B(\delta_1 e_1, r)$. 
We have $p((\delta_1+\Delta) e_1) - p(\delta_1 e_1) = \frac{1}{V} \left( \Vol(D) - \Vol(A) \right)$, and we will show that there is an isometry $v$ such that $v(A) \subset D$.

Let $v$ be the reflection across the hyperplane $\left\{ x \in \mathbb{R}^d \mid x_1 = \delta_1 + \frac{\Delta}{2} \right\}$.
Let $x = (x_1,\dots,x_d) \in A$.
It verifies the following equations:
\begin{empheq}[left=\empheqlbrace]{align}
(x_1-\delta_1)^2 + x_2^2 + x_3^2 + \dots + x_d^2 \leq r^2 \label{eq21} \\
x_1 > 0 \label{eq22}\\
x_2^2 + \dots + x_d^2 \leq x_1^2 \tan^2(\theta) \label{eq23}\\
(x_1 -\delta_1 - \Delta)^2 + x_2^2 + x_3^2 + \dots + x_d^2 > r^2 \label{eq24}   
\end{empheq}
$v(x_1,\dots,x_d) = (2\delta_1 + \Delta - x_1,x_2,\dots,x_d) = \tilde{x}$.

First of all, subtracting \Cref{eq24} from \Cref{eq21} gives:
\begin{align}
    (x_1-\delta_1-\Delta)^2 > (x_1-\delta_1)^2 &\Rightarrow \Delta^2 - 2 \Delta(x_1-\delta_1) > 0 \\
    &\Rightarrow x_1 < \delta_1 + \frac{\Delta}{2}
\end{align}

Let us show $\tilde{x} \in D$.
\begin{align}
    \tilde{x_1}^2\tan^2(\theta) &= (2\delta_1 + \Delta - x_1)^2\tan^2(\theta) \\
    &\geq \left( 2\delta_1 + \Delta - \delta_1 - \frac{\Delta}{2} \right)^2 \tan^2(\theta) \\
    &= \left( \delta_1 + \frac{\Delta}{2} \right)^2 \tan^2(\theta) \\
    &\geq x_1^2 \tan^2(\theta) \\
    &\geq x_2^2 + \dots + x_d^2 \\
    &= \tilde{x_2}^2 + \dots + \tilde{x_d}^2
\end{align}

Hence $\tilde{x} \in C(0, \theta)$. Then:
\begin{align}
    (\tilde{x_1}-\delta_1 - \Delta)^2 + \tilde{x_2}^2 + \dots + \tilde{x_d}^2 &= (x_1-\delta_1)^2 + x_2^2 + \dots + x_d^2 \\
    &\leq r^2
\end{align}

Hence $\tilde{x} \in B((\delta_1 + \Delta)e_1,r)$. Finally,
\begin{align}
    (\tilde{x_1}-\delta_1)^2 + \tilde{x_2}^2 + \dots + \tilde{x_d}^2 &= (x_1-\delta_1 - \Delta)^2 + x_2^2 + \dots + x_d^2 \\
    &> r^2
\end{align}
and we have $\tilde{x} \notin B(\delta_1 e_,r)$. We have thus shown that $\Vol(D) \geq \Vol(A)$, and so the attack is increasing in $\delta_1$.

We will now show that the increase is strict. For that, we show that there exists points in $D$ whose image by $v$ is not in $A$.

Recall that $D = B(\delta_1+\Delta,r) \cap C(0,\theta) \setminus B(\delta_1,r)$ is defined by the following set of equations:

\begin{empheq}[left=\empheqlbrace]{align}
(x_1-\delta_1 - \Delta)^2 + x_2^2 + x_3^2 + \dots + x_d^2 \leq r^2 \label{eq31} \\
x_1 > 0 \label{eq32}\\
x_2^2 + \dots + x_d^2 \leq x_1^2 \tan^2(\theta) \label{eq33}\\
(x_1 -\delta_1)^2 + x_2^2 + x_3^2 + \dots + x_d^2 > r^2 \label{eq34}   
\end{empheq}

Let us reason by contradiction, and assume that $v(D) \subset A$ 

This means that for all points verifying the previous set of equations, we also have $ x_2^2+\dots + x_d^2 \leq \tilde{x}_1^2 \tan^2(\theta)$, i.e.

\begin{equation}
    \label{eq35}
    x_2^2+\dots + x_d^2 \leq (2\delta_1 + \Delta - x_1)^2 \tan^2(\theta)
\end{equation}

Let us define $u = x_1 - \delta$ and $b=\delta+\Delta$. Combining \cref{eq35} and \cref{eq34} gives:

\begin{align*}
r^2 &\leq (x_1-\delta)^2 + (2\delta+\Delta-x_1)^2 \tan^2(\theta)\\
    &\leq u^2 + (b- u)^2 \tan^2(\theta) \\
    &\leq u^2 + (b^2 + u^2 -2bu)\tan^2(\theta) \\
\end{align*}
This implies:
\begin{align*}
    & (1+\tan^2(\theta))u^2 -2b\tan^2(\theta)u + b^2-r^2 > 0 \\
    &\Rightarrow b^2 - r^2 > b^2\frac{tan^2(\theta)}{1 + \tan^2(\theta)}\\
    &\Rightarrow r^2 < b^2 (1-\sin^2(\theta)\tan^2(\theta)) \\
    &\Rightarrow r^2 < \delta^2 (1-\tan^2(\theta)) \\
    &\Rightarrow r^2 < \delta^2
\end{align*}
Which is a contradiction since we consider attacks of size $\delta = \epsilon < r$.

We have shown that any component of the attack that is orthogonal to $e_1$ is detrimental to the attack, and that an increase along $e_1$ benefits the attack. It follows that the optimal attack of size at most $\epsilon$ is $\epsilon e_1$.
\end{proof}

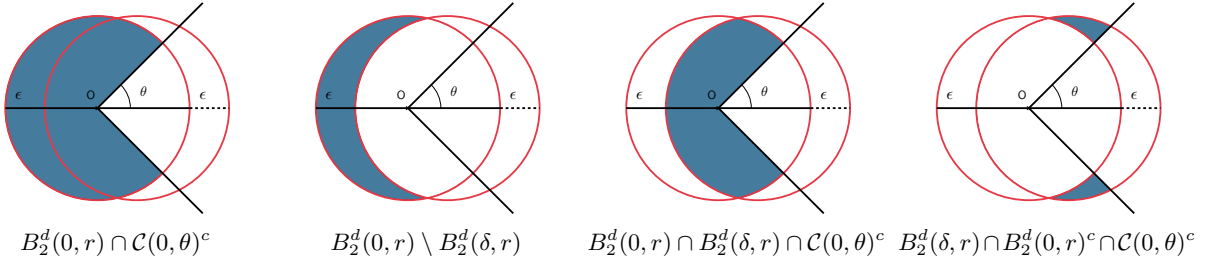
\begin{figure}[th]
  \centering
  \hfill
  \begin{subfigure}[h]{0.24\textwidth}
    \centering
    \scalebox{0.35}{%
      \def\centerarc[#1](#2)(#3:#4:#5)%
    { \draw[#1] ($(#2)+({#5*cos(#3)},{#5*sin(#3)})$) arc (#3:#4:#5); }
\tikzset{%
  >={Latex[width=0mm,length=0mm]},
  base/.style = {fill=white, font=\sffamily},
  dot/.style = {circle, fill=black, minimum size=#1, inner sep=0pt, outer sep=0pt},
  cross/.style = {cross out, draw=black, minimum size=4pt, inner sep=0pt, outer sep=0pt}
}
\begin{tikzpicture}[every node/.style=base, align=center, scale=5]

  \clip (-0.3, -0.3) rectangle (1.6, 1.3);

  \begin{scope}
    \clip (0.5, 0.5) -- (0.5+0.8, 0.5+0.8) -- (-0.5, 1.5) -- (-0.5, -0.5) -- (0.5+0.8, 0.5-0.8) -- (0.5, 0.5);
    \draw[fill=myblue,even odd rule] (0.5, 0.5) circle (0.7) ;
  \end{scope}
   
   \draw[color=myred, line width=2pt] (0.5, 0.5) circle (0.7);
   \draw[color=myred, line width=2pt] (0.8, 0.5) circle (0.7);

   \draw[color=black, line width=2pt] (-0.2, 0.5) -- (1.2, 0.5);
   \draw[color=black, line width=2pt, dashed] (1.2, 0.5) -- (1.5, 0.5);

   \draw (.5,.5) node[cross,rotate=0] {};

   \node[draw=none, fill=white, text opacity=1, fill opacity=0., scale=1.8] at (-0.1,  0.6) {$\epsilon$};
   \node[draw=none, fill=white, text opacity=1, fill opacity=0., scale=1.8] at (1.3, 0.6) {$\epsilon$};
   \node[draw=none, fill=white, text opacity=1, fill opacity=0., scale=1.5] at (0.85,  0.63) {$\theta$};

   \node[draw=none, fill=white, text opacity=1, fill opacity=0., scale=1.3] at (0.45,  0.6) {O};

   \draw[color=black, line width=2pt] (0.5, 0.5) -- (0.5+0.8, 0.5+0.8);
   \draw[color=black, line width=2pt] (0.5, 0.5) -- (0.5+0.8, 0.5-0.8);

   \centerarc[black](0.5, 0.5)(0:45:0.25);

\end{tikzpicture}%
    }%
    \caption*{$B_2^d(0, r) \cap \mathcal{C}(0, \theta)^c$}
  \end{subfigure}
  \hfill
  \begin{subfigure}[h]{0.24\textwidth}
    \centering
    \scalebox{0.35}{%
      \def\centerarc[#1](#2)(#3:#4:#5)%
    { \draw[#1] ($(#2)+({#5*cos(#3)},{#5*sin(#3)})$) arc (#3:#4:#5); }
\tikzset{%
  >={Latex[width=0mm,length=0mm]},
  base/.style = {fill=white, font=\sffamily},
  dot/.style = {circle, fill=black, minimum size=#1, inner sep=0pt, outer sep=0pt},
  cross/.style = {cross out, draw=black, minimum size=4pt, inner sep=0pt, outer sep=0pt}
}
\begin{tikzpicture}[every node/.style=base, align=center, scale=5]

  \clip (-0.3, -0.3) rectangle (1.6, 1.3);

  \begin{scope}
    \clip (0.5, 0.5) -- (0.5+0.8, 0.5+0.8) -- (-0.5, 1.5) -- (-0.5, -0.5) -- (0.5+0.8, 0.5-0.8) -- (0.5, 0.5);
    \clip (0.5, 0.5) circle (0.7);
    \draw[fill=myblue,even odd rule] (0.5, 0.5) circle (0.7) (0.8, 0.5) circle (0.7);
  \end{scope}
   
   \draw[color=myred, line width=2pt] (0.5, 0.5) circle (0.7);
   \draw[color=myred, line width=2pt] (0.8, 0.5) circle (0.7);

   \draw[color=black, line width=2pt] (-0.2, 0.5) -- (1.2, 0.5);
   \draw[color=black, line width=2pt, dashed] (1.2, 0.5) -- (1.5, 0.5);

   \draw (.5,.5) node[cross,rotate=0] {};

   \node[draw=none, fill=white, text opacity=1, fill opacity=0., scale=1.8] at (-0.1,  0.6) {$\epsilon$};
   \node[draw=none, fill=white, text opacity=1, fill opacity=0., scale=1.8] at (1.3, 0.6) {$\epsilon$};
   \node[draw=none, fill=white, text opacity=1, fill opacity=0., scale=1.5] at (0.85,  0.63) {$\theta$};

   \node[draw=none, fill=white, text opacity=1, fill opacity=0., scale=1.3] at (0.45,  0.6) {O};

   \draw[color=black, line width=2pt] (0.5, 0.5) -- (0.5+0.8, 0.5+0.8);
   \draw[color=black, line width=2pt] (0.5, 0.5) -- (0.5+0.8, 0.5-0.8);

   \centerarc[black](0.5, 0.5)(0:45:0.25);

\end{tikzpicture}%
    }%
    \caption*{$B_2^d(0, r) \setminus B_2^d(\delta, r)$}
  \end{subfigure}
  \hfill
  \begin{subfigure}[h]{0.24\textwidth}
    \centering
    \scalebox{0.35}{%
      \def\centerarc[#1](#2)(#3:#4:#5)%
    { \draw[#1] ($(#2)+({#5*cos(#3)},{#5*sin(#3)})$) arc (#3:#4:#5); }
\tikzset{%
  >={Latex[width=0mm,length=0mm]},
  base/.style = {fill=white, font=\sffamily},
  dot/.style = {circle, fill=black, minimum size=#1, inner sep=0pt, outer sep=0pt},
  cross/.style = {cross out, draw=black, minimum size=4pt, inner sep=0pt, outer sep=0pt}
}
\begin{tikzpicture}[every node/.style=base, align=center, scale=5]

  \clip (-0.3, -0.3) rectangle (1.6, 1.3);

  \begin{scope}
    \clip (0.5, 0.5) -- (0.5+0.8, 0.5+0.8) -- (-0.5, 1.5) -- (-0.5, -0.5) -- (0.5+0.8, 0.5-0.8) -- (0.5, 0.5);
    \clip (0.8, 0.5) circle (0.7);
    \draw[fill=myblue,even odd rule] (0.5, 0.5) circle (0.7) ;
  \end{scope}

  \draw[color=myred, line width=2pt] (0.5, 0.5) circle (0.7);
  \draw[color=myred, line width=2pt] (0.8, 0.5) circle (0.7);

  \draw[color=black, line width=2pt] (-0.2, 0.5) -- (1.2, 0.5);
  \draw[color=black, line width=2pt, dashed] (1.2, 0.5) -- (1.5, 0.5);

  \draw (.5,.5) node[cross,rotate=0] {};

  \node[draw=none, fill=white, text opacity=1, fill opacity=0., scale=1.8] at (-0.1,  0.6) {$\epsilon$};
  \node[draw=none, fill=white, text opacity=1, fill opacity=0., scale=1.8] at (1.3, 0.6) {$\epsilon$};
  \node[draw=none, fill=white, text opacity=1, fill opacity=0., scale=1.5] at (0.85,  0.63) {$\theta$};

  \node[draw=none, fill=white, text opacity=1, fill opacity=0., scale=1.3] at (0.45,  0.6) {O};

  \draw[color=black, line width=2pt] (0.5, 0.5) -- (0.5+0.8, 0.5+0.8);
  \draw[color=black, line width=2pt] (0.5, 0.5) -- (0.5+0.8, 0.5-0.8);

  \centerarc[black](0.5, 0.5)(0:45:0.25);

\end{tikzpicture}%
    }%
    \caption*{$B_2^d(0, r) \cap B_2^d(\delta, r) \cap \mathcal{C}(0, \theta)^c$}
  \end{subfigure}
  \hfill
  \begin{subfigure}[h]{0.24\textwidth}
    \centering
    \scalebox{0.35}{%
      \def\centerarc[#1](#2)(#3:#4:#5)%
    { \draw[#1] ($(#2)+({#5*cos(#3)},{#5*sin(#3)})$) arc (#3:#4:#5); }
\tikzset{%
  >={Latex[width=0mm,length=0mm]},
  base/.style = {fill=white, font=\sffamily},
  dot/.style = {circle, fill=black, minimum size=#1, inner sep=0pt, outer sep=0pt},
  cross/.style = {cross out, draw=black, minimum size=4pt, inner sep=0pt, outer sep=0pt}
}
\begin{tikzpicture}[every node/.style=base, align=center, scale=5]

  \clip (-0.3, -0.3) rectangle (1.6, 1.3);

  \begin{scope}
    \clip (0.5, 0.5) -- (0.5+0.8, 0.5+0.8) -- (0., 1.5) -- (0, -0.5) -- (0.5+0.8, 0.5-0.8) -- (0.5, 0.5);
    \clip (0.8, 0.5) circle (0.7); 
    \draw[fill=myblue,even odd rule] (0.8, 0.5) circle (0.7) (0.5, 0.5) circle (0.7) ;
  \end{scope}
   
   \draw[color=myred, line width=2pt] (0.5, 0.5) circle (0.7);
   \draw[color=myred, line width=2pt] (0.8, 0.5) circle (0.7);

   \draw[color=black, line width=2pt] (-0.2, 0.5) -- (1.2, 0.5);
   \draw[color=black, line width=2pt, dashed] (1.2, 0.5) -- (1.5, 0.5);

   \draw (.5,.5) node[cross,rotate=0] {};

   \node[draw=none, fill=white, text opacity=1, fill opacity=0., scale=1.8] at (-0.1,  0.6) {$\epsilon$};
   \node[draw=none, fill=white, text opacity=1, fill opacity=0., scale=1.8] at (1.3, 0.6) {$\epsilon$};
   \node[draw=none, fill=white, text opacity=1, fill opacity=0., scale=1.5] at (0.85,  0.63) {$\theta$};

   \node[draw=none, fill=white, text opacity=1, fill opacity=0., scale=1.3] at (0.45,  0.6) {O};

   \draw[color=black, line width=2pt] (0.5, 0.5) -- (0.5+0.8, 0.5+0.8);
   \draw[color=black, line width=2pt] (0.5, 0.5) -- (0.5+0.8, 0.5-0.8);

   \centerarc[black](0.5, 0.5)(0:45:0.25);

\end{tikzpicture}%
    }%
    \caption*{$B_2^d(\delta, r) \cap B_2^d(0, r)^c \cap \mathcal{C}(0, \theta)^c$}
  \end{subfigure}
  \hfill
  \caption{Illustration of proof of \Cref{theorem:underestimate_cohen}. The worst-case classifier using only the information $p(\theta)$ assumes that the zone in the second figureis entirely lost, whereas for the perfect certificate, the blue zone in the fourth figure is not lost. That zone grows as $\theta$ shrinks.}
  \label{figure:underestimate_cohen_proof}
\end{figure}

\begin{proof}[Proof of \Cref{theorem:underestimate_cohen}]

Let $r > 0$, $0 < \epsilon \leq r$, $\delta = [\epsilon, 0, \dots, 0]^\top \in \mathbb{R}^d$, $\theta \in [0, \theta_m]$ with $\theta_m = \arccos(\frac{\epsilon}{2r})$.
Let $\mathcal{C}(0,\theta)$ be a cone of revolution of peak 0, axis $e_1$ and angle $\theta$.
Let us consider all functions $h_\theta$ whose decision boundary is the cone of revolution $\mathcal{C}(0,\theta)$.
The probability $p(x, h_\theta, q_0)$ of returning class 1 at the point $0$ for the classifier $h_\theta$ after smoothing by $q_0$ is:
\begin{align}
  p(x, h_\theta, q_0) &= \int\limits_{\mathbb{R}^d} \frac{\mathbb{1}\left\{ x \in B_2^d(0,r)\right\}}{\Vol(B_2^d(0, r))}
  \mathbb{1}\left\{x \in \mathcal{C}(0, \theta)^c\right\} \diff x\\
  &= \frac{\Vol(B_2^d(0, r) \cap \mathcal{C}(0, \theta)^c)}{\Vol(B_2^d(0, r))} 
\end{align} 

Single-noise certificates uses the fact that in the worst-case scenario, all the volume lost during the translation was in class 1, and all the volume gained is in the class $0$ (see proof of \citet{yang2020randomized} [Theorem I.19]) This gives:
\begin{align}
  \NC(h_\theta, q_0, x, \epsilon, \{q_0\} ) &= p(x, h_\theta, q_0) - \frac{1}{V} \left( \Vol(B_2^d(0, r) \setminus B_2^d(\delta, r)) \right) \label{eq:worstCase}\\
  &= \frac{1}{V} \left( \Vol(B_2^d(0, r) \cap \mathcal{C}(0, \theta)^c) - \left( \Vol(B_2^d(0, r) \setminus B_2^d(\delta, r)) \right) \right) \\
  &= \frac{1}{V} \Vol( B_2^d(0, r) \cap B_2^d(\delta, r) \cap \mathcal{C}(0, \theta)^c)
\end{align} 
In \Cref{eq:worstCase}, we say that in the worst case scenario all the volume lost during the translation was in class 1, and all volume gained was in class 0, hence we loose everything outside the intersection.

This corresponds to \Cref{figure:lemma_fig2}: the zone that is preserved after translation for the noise certificate is the blue one in the third figure. We will show that the perfect certificate also preserves the blue zone in the fourth figure.

Since by \Cref{lemma:optimal_attack_e1} the optimal attack against the cone is the translation along its axis, the perfect certificate for the probability $p$ will be defined under the attack $\delta$:
\begin{equation}
  \PC(h_\theta, q_0, x, \epsilon) = \frac{1}{V} \Vol(B_2^d(\delta, r) \cap \mathcal{C}(0, \theta)^c)
\end{equation}
The difference between the perfect certificate and the single-noise based certificate (as in \Cref{definition:diff_pc_nc}) is:
\begin{align}
  \nu(h_{\theta}, q_0, x, \epsilon, \left\{ q_0 \right\}) &= \frac{1}{V} \left( \Vol(B_2^d(\delta, r) \cap \mathcal{C}(0, \theta)^c) - \Vol( B_2^d(0, r) \cap B_2^d(\delta, r) \cap \mathcal{C}(0, \theta)^c) \right) \\
  &= \frac{1}{V} \Vol( B_2^d(\delta, r) \cap \mathcal{C}(0, \theta)^c \setminus \left( B_2^d(0, r) \cap B_2^d(\delta, r) \cap \mathcal{C}(0, \theta)^c \right) ) \\
  &= \frac{1}{V} \Vol( B_2^d(\delta, r) \cap \mathcal{C}(0, \theta)^c \cap \left( B_2^d(0, r) \cap B_2^d(\delta, r) \cap \mathcal{C}(0, \theta)^c \right)^{c} ) \\
  &= \frac{1}{V} \Vol( B_2^d(\delta, r) \cap \mathcal{C}(0, \theta)^c \cap \left( B_2^d(0, r)^c \cup B_2^d(\delta, r)^c \cup \mathcal{C}(0, \theta) \right) ) \\
  &= \frac{1}{V} \Vol( B_2^d(\delta, r) \cap B_2^d(0, r)^c \cap \mathcal{C}(0, \theta)^c ) \\
  &= \frac{A}{V}\int\limits_{x=0}^{\infty} \int\limits_{\rho=x \tan(\theta)}^{\infty}
  \mathbb{1}\left\{ (x-\epsilon)^2 + \rho^2 \leq r^2 \right\} \mathbb{1} \left\{x^2+\rho^2>r^2 \right\} \rho^{d-2}\diff x \diff \rho
\end{align}

where $A=\int\limits_{\substack{\phi_1, \dots, \phi_{d-3} \in [-\pi, \pi] \\ \phi_{d-2} \in [0, 2\pi]}} \prod_{k = 1}^{d-2} \sin^k \phi_{d-1-k} \diff \phi_{1} \dots \diff \phi_{d-2}$.

It follows that $\nu$ is a continuous function with respect to $\theta \in \left[ 0, \theta_{m} \right]$. 
It is decreasing, since $\mathcal{C}(0, \theta_1) \subset \mathcal{C}(0, \theta_2)$ when $\theta_1 < \theta_2$.

Furthermore, when $\theta = 0$:
\begin{align}
  \nu(h_{0}, q_0, x, \epsilon, \left\{ q_0 \right\}) &= \frac{1}{V} \Vol( B_2^d(\delta, r) \cap B_2^d(0, r)^c ) \\
  &= \frac{1}{V} \Vol( B_2^d(\delta, r) \setminus B_2^d(0, r)) \\
  &= \frac{1}{V} \Vol( B_2^d(\delta, r) ) - \Vol( B_2^d(0, r) \cap B_2^d(\delta, r) ) \label{eq:vol_interesction_ball_1} \\
  &= \frac{1}{V} \left( \Vol( B_2^d(\delta, r) - 2 \Vol(\SphereCap(r-\frac{\epsilon}{2}, r, d)) \right) \label{eq:vol_interesction_ball_2} \\
  &= 1- I_{1 - \left(\frac{\epsilon}{2r}\right)^2} \left( \frac{d+1}{2}, \frac{1}{2} \right) 
\end{align}
where the step from \Cref{eq:vol_interesction_ball_1} to \Cref{eq:vol_interesction_ball_2} is due because the intersection of both spheres is the union of two spherical caps.

Moreover, from \Cref{lemma:limitBetaFunction}, we have $\nu(h_{0}, q_0, x, \epsilon, \left\{ q_0 \right\}) \xrightarrow[d \to \infty]{} 1$:

And, when $\theta = \theta_m$, we are going to prove that $\nu(h_{\theta_m}, q_0, x, \epsilon, \left\{ q_0 \right\}) = 0$. 
Equivalently, we want to show that the set defined by:
\begin{equation}
\left\{ (x, \rho) \in \mathbb{R} \mathrel{\big|} x < \frac{\epsilon}{2} \mathrel{\big|} (x-\epsilon)^2 + \rho^2 \leq r^2 \mathrel{\big|} \rho > x \tan \theta_m \right\}
\end{equation}
is an empty set.
Let $(x,\rho)$ in this set. We have:
\begin{align}
  \rho &> x \tan( \arccos\left(\frac{\epsilon}{2r}\right) ) \\ 
   &= \frac{2rx}{\epsilon} \sqrt{1 - \frac{\epsilon^2}{4r^2}}
\end{align}
due to the equality: $\tan(\arccos(x)) = \frac{\sqrt{1-x^2}}{x}$.
Then, we have:
\begin{align}
  r^2 &\geq (x-\epsilon)^2 + \rho^2 \\
  &\geq x^2 - 2 x \epsilon + \epsilon^2 + \frac{4 r^2 x^2}{\epsilon^2} \left( 1 - \frac{\epsilon^2}{4r^2} \right) \\ 
  &= x^2 - 2 x \epsilon + \epsilon^2 + \frac{4 r^2 x^2}{\epsilon^2} - x^2 \\
  &= \frac{4 r^2}{\epsilon^2} x^2  - 2 x \epsilon + \epsilon^2
\end{align}
Hence we have:
\begin{equation}
\label{equation:contradiction}
  \frac{4 r^2}{\epsilon^2} x^2  - 2 x \epsilon + \epsilon^2 - r^2 \leq 0 \quad \text{and } \quad x \leq \frac{\epsilon}{2}
\end{equation}

But the minimum of the right hand side is: $\frac{\epsilon^3}{4r^2} \leq \frac{\epsilon}{2}$ because $r \leq \epsilon$.
Therefore, the r.h.s is increasing on the interval $[\frac{\epsilon}{2}, \infty]$ and is equal to 0 when $x = \frac{\epsilon}{2}$, which proves that no point verifies \Cref{equation:contradiction}. That allows us to conclude that: $\nu(h_{\theta_m}, q_0, x, \epsilon, \left\{ q_0 \right\}) = 0$.

\end{proof}

\subsection{Proof of \cref{proposition:findingUnderestimate}}
\findingUnderestimatePoints*

\begin{proof}

Recall the definition of the noised-based certificates (see~\Cref{definition:noised_based_certificate}).
When the infimum over $\mathcal{G}_{\mathcal{Q}}$ is attained by some $g$, we call $g$ a $\mathcal{Q}$-worst case classifier.
Now, let us denote $q_r$, the uniform distribution of radius $r>0$, let $r_1 > 0$ and let $D_{r_1} = \left\{z \in \mathcal{X} \mid g_{r_1}(z)=1\right\}$, a half-space. 
As we saw in the proof of \Cref{theorem:underestimate_cohen}, for any $\delta$ of norm $\epsilon$, the $\left\{q_{r_1}\right\}$-worst classifier $g_{r_1}$ has a decision region $D_{r_1}$ that is entirely contained in $B_2^d(x,r_1)$.
Hence for any $r > r_1$,
\begin{align}
    p(x,g_{r_1},q_r) &= \mathbb{P}\left[ U(x,r) \in D_{r_1} \right] \\ &= \frac{\Vol(B_2^d(0,r) \cap D_{r_1})}{\Vol(B_2^d(x,r))} \\
    &= \frac{\Vol(B_2^d(x,r_1) \cap D_{r_1})}{\Vol(B_2^d(x,r))}
\end{align}
because $B_2^d(x,r) \cap D_{r_1} \subset B_2^d(x,r_1)$. Since $r_1$ is constant, this is a decreasing function in $r$.

A similar proof works for 2-piecewise linear decision boundaries.
\end{proof}

\subsection{Proof of \Cref{proposition:successive_noises}}
\successivenoises*

\begin{figure}[th]
  \centering
  \hfill
  \begin{subfigure}[h]{0.49\textwidth}
    \centering
    \scalebox{0.35}{%
      \def\centerarc[#1](#2)(#3:#4:#5)%
    { \draw[#1] ($(#2)+({#5*cos(#3)},{#5*sin(#3)})$) arc (#3:#4:#5); }
\tikzset{%
  >={Latex[width=0mm,length=0mm]},
  base/.style = {fill=white, font=\sffamily},
  dot/.style = {circle, fill=black, minimum size=#1, inner sep=0pt, outer sep=0pt},
  cross/.style = {cross out, draw=black, minimum size=7pt, inner sep=0pt, outer sep=0pt}
}
\begin{tikzpicture}[every node/.style=base, align=center, scale=5]

  \clip (-0.4, -0.3) rectangle (1.3, 1.3);

  \begin{scope}
    \clip (0.5, -0.3) -- (0.5, 1.3) -- (1.2, 1.3) -- (1.2, -0.3) -- (0.5, -0.3);
    \draw[fill=myyellow, even odd rule, draw=none]
     (0.4, 0.5) circle (0.5) (0.4, 0.5) circle (0.7) ;
  \end{scope}

  \draw[color=myred, line width=2pt] (0.4, 0.5) circle (0.5);
  \draw[color=myred, line width=2pt] (0.4, 0.5) circle (0.7);

  \draw[color=black, line width=2pt] (0.5, -0.4) -- (0.5, 1.4);
  \draw[color=black, line width=2pt] (0.5,  0.5) -- (0.5+0.8, 0.5+0.8);
  \draw[color=black, line width=2pt] (0.5,  0.5) -- (0.5+0.8, 0.5-0.8);

  \draw[color=black, line width=2pt] (-0.4, 0.5) -- (1.2, 0.5);
  \draw (0.4,0.5) node[cross,rotate=0] {};
  \node[draw=none, fill=white, text opacity=1, fill opacity=0., scale=1.5] at (0.8, 0.57) {$\theta$};
  \node[draw=none, fill=white, text opacity=1, fill opacity=0., scale=1.3] at (0.4, 0.6) {O};

  \centerarc[black](0.4,0.5)(0:33:0.35);

\end{tikzpicture}%
    }%
    \caption*{$\mathcal{A}(r_1, r_2) = H(c) \cap B_2^d(0,r_2) \setminus B_2^d(0,r_1)$}
  \end{subfigure}
  \hfill
  \begin{subfigure}[h]{0.49\textwidth}
    \centering
    \scalebox{0.35}{%
      \def\centerarc[#1](#2)(#3:#4:#5)%
    { \draw[#1] ($(#2)+({#5*cos(#3)},{#5*sin(#3)})$) arc (#3:#4:#5); }
\tikzset{%
  >={Latex[width=0mm,length=0mm]},
  base/.style = {fill=white, font=\sffamily},
  dot/.style = {circle, fill=black, minimum size=#1, inner sep=0pt, outer sep=0pt},
  cross/.style = {cross out, draw=black, minimum size=7pt, inner sep=0pt, outer sep=0pt}
}
\begin{tikzpicture}[every node/.style=base, align=center, scale=5]

  \clip (-0.4, -0.3) rectangle (1.3, 1.3);

  \begin{scope}
    \clip (0.5, 0.5) -- (0.5+0.8, 0.5+0.8) -- (0.5+0.8, 0.5-0.8); 
    \draw[fill=myyellow, even odd rule, draw=none]
     (0.4, 0.5) circle (0.5) (0.4, 0.5) circle (0.7) ;
  \end{scope}

  \draw[color=myred, line width=2pt] (0.4, 0.5) circle (0.5);
  \draw[color=myred, line width=2pt] (0.4, 0.5) circle (0.7);

  \draw[color=black, line width=2pt] (0.5, -0.4) -- (0.5, 1.4);
  \draw[color=black, line width=2pt] (0.5,  0.5) -- (0.5+0.8, 0.5+0.8);
  \draw[color=black, line width=2pt] (0.5,  0.5) -- (0.5+0.8, 0.5-0.8);

  \draw[color=black, line width=2pt] (-0.4, 0.5) -- (1.2, 0.5);
  \draw (0.4,0.5) node[cross,rotate=0] {};
  \node[draw=none, fill=white, text opacity=1, fill opacity=0., scale=1.5] at (0.8, 0.57) {$\theta$};
  \node[draw=none, fill=white, text opacity=1, fill opacity=0., scale=1.3] at (0.4, 0.6) {O};

  \centerarc[black](0.4,0.5)(0:33:0.35);

\end{tikzpicture}%
    }%
    \caption*{$\mathcal{B}_{\theta}(r_1,r_2) = \mathcal{C}(c, \theta) \cap B_2^d(0,r_2) \setminus B_2^d(0,r_1)$}
  \end{subfigure}
  \hfill
  \caption{Illustration of the proof of \Cref{lemma:successive_noises}.}
  \label{figure:successive_noises_proof}
\end{figure}
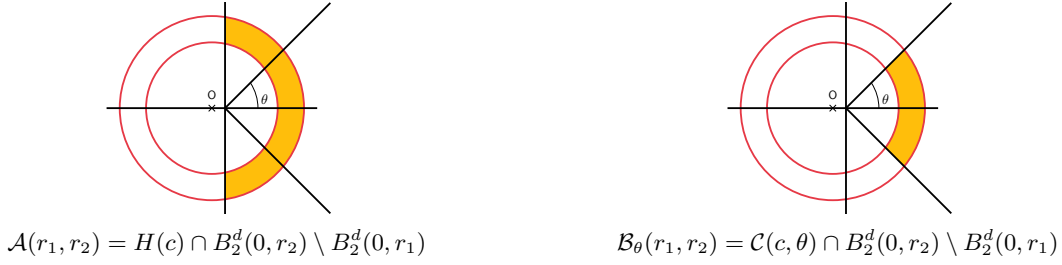

\begin{definition}[{\bf Volume growth for a decision boundary}]
Let $B_2^d$, the $\ell_2$-ball in dimension $d$, $\Acal$ a set and $r_1 > r_2 \geq 0$. We define the volume growth of $\Acal$ from $r_1$ to $r_2$ as:
\begin{equation*}
\Delta V (\Acal, r_1,r_2) = \Vol \left( \Acal \cap B_2^d(0,r_2) \right) - \Vol \left( \Acal \cap B_2^d(0,r_1) \right)
\end{equation*}
\end{definition}

\begin{lemma}[{\bf Growth function for concentric noises}]
\label{lemma:successive_noises}
Let $c \geq 0$, $r_2 > r_1 > c$. Then $\Delta V (C(c, \theta), r_1, r_2)$ is a continuous, increasing function of $\theta$, that is a bijection from  $\left[0, \pi \right]$ to $\left[ 0, \Delta V (H(c), r_1, r_2) \right]$.

The same result holds for 2-piecewise linear sets of parameter $\theta$.
\end{lemma}

\begin{proof}[Proof of \Cref{lemma:successive_noises}]
Let $r_2 > r_1 > 0$, $c >0$.
Let $\mathcal{C}(c, \theta)$ be the cone of revolution of peak $c$ and angle $\theta$. Let $\mathcal{B}_{\theta}(r_1,r_2) = \mathcal{C}(c, \theta) \cap B_2^d(0,r_2) \setminus B_2^d(0,r_1)$ and $\mathcal{A}(r_1, r_2) = H(c) \cap B_2^d(0,r_2) \setminus B_2^d(0,r_1) $. 

First note that $\Vol (\mathcal{B}_{\theta}(r_1,r_2)) = \Vol(\mathcal{C}(c, \theta) \cap B_2^d(0,r_2)) - \Vol(\mathcal{C}(c, \theta) \cap B_2^d(0,r_1)) = \Delta V(\mathcal{C}(c,\theta),r_1,r_2)$, and similarly $\Vol (\mathcal{A}(r_1, r_2)) = \Delta V(H(c),r_1,r_2)$

To compute the volume of $\mathcal{B}_{\theta}(r_1,r_2)$, we must cut the integral into three zone: where the bound is the cone, and where it is the surface of either of the balls.
There exist a constant $K$ (dependent on the dimension, and containing the integration in all the variables $\phi_i$ in hyper-spherical coordinates) such as:
\begin{align}
  \Vol( \mathcal{B}_{\theta}(r_1,r_2) ) &=  \Vol(\mathcal{C}(c, \theta) \cap B_2^d(0,r_2) \setminus B_2^d(0,r_1)) \\
  &= \frac{1}{K} \left[ \int_{x = r_1 \cos \theta}^{r_1} \int_{\sqrt{r_1^2 - x^2}}^{x \tan \theta} \rho^{d-2} \diff \rho \diff x  + \right. \notag \\
  &\quad\quad\quad \quad \quad \left. \int_{r_1}^{r_2 \cos \theta} \int_{\rho = 0}^{x \tan \theta} \rho^{d-2} \diff \rho \diff x + \int_{r_2 \cos \theta}^{r_2} \int_{\rho = 0}^{\sqrt{r_2^2 - x^2}} \rho^{d-2} \diff \rho \diff x \right]
\end{align}

It follows that $\Vol (\mathcal{B}_\theta(r_1, r_2))$ is a continuous function of $\theta$.

Furthermore, if $\theta_2 > \theta_1$, then $\mathcal{C}(c,\theta_1) \subset \mathcal{C}(c,\theta_2)$, so $\Vol (\mathcal{B}_{\theta_1}(r_1, r_2)) \leq \Vol (\mathcal{B}_{\theta_2}(r_1, r_2))$, and the function is increasing.

For $\theta=0$, $\Vol (\mathcal{B}_{\theta}(r_1, r_2)) = 0$, and for $\theta=\frac{\pi}{2}$, $\mathcal{B}_{\theta_1}(r_1, r_2)) = \mathcal{A}(r_1,r_2)$. Hence the result.

\end{proof}

\begin{proof}[Proof of \Cref{proposition:successive_noises}]
This is an immediate consequence of \cref{lemma:successive_noises} : From the information of two noises, we can perfectly identify the parameter $\theta$, and thus compute the perfect certificate as $\mathbb{P}_{X \sim q_0(x)}\left[X + \epsilon e_1 \in \mathcal{C}(0,\theta) \right]$.
\end{proof}

\section{Proofs of \Cref{section:neyman_pearson}}

\subsection{Proof of \Cref{lemma:generalized_neyman_pearson}} 
\generalizedneymanpearson*

\begin{proof}[Proof of \Cref{lemma:generalized_neyman_pearson}]
By definition of $\Phi_\mathcal{K}$, we have:
\begin{equation}
   \int (\Phi - \Phi_\mathcal{K}) (q_0-\sum\limits_{k=1}^n k_i q_i ) \diff \mu \geq 0
\end{equation}
since the integrand is always positive. Hence: 
\begin{equation}
  \int (\Phi - \Phi_\mathcal{K}) q_0 \diff \mu \geq \sum\limits_{k=1}^n k_i \int (\Phi - \Phi_\mathcal{K}) q_i d\mu
\end{equation}

Since $\int (\Phi - \Phi_\mathcal{K}) q_i \diff \mu \geq 0$, we have:
\begin{equation}
  \int (\Phi - \Phi_\mathcal{K}) q_0 \diff \mu \geq 0
\end{equation}
which is the desired result.
\end{proof}

\subsection{Proof of \Cref{theorem:approx_theorem}}
\approximationtheorem*

\begin{proof}[Proof of \Cref{theorem:approx_theorem}]

Let $n>0$, and some $x \in \mathcal{X}$.
We can construct a grid of $(n(r+ {2} \epsilon))^d$ disjoint squares of side size $\frac{1}{n}$, of the form $\left[ \frac{a_1}{n}, \frac{a_1 + 1}{n} \right[ \times \dots \times \left[ \frac{a_d}{n}, \frac{a_d + 1}{n} \right[$ (except the ones on the border of the ball that are closed) that will cover the ball $B_\infty^d(x,r)$, as well as its translation by $\epsilon$ in any direction.

Let us call the squares in this grid $A_j$ for $j=1 \dots m$ and $m=(\frac{d+{2}\epsilon}{n})^d$.
They all have the same volume $V_n = (\frac{1}{n})^d$.

The idea of this proof is to construct a noise-based classifier that returns 1 only on the squares that are entirely contained in the decision region, \ie. a strict underestimate of the true classifier.
As the grid gets thinner, the approximation will then converge to the true classifier as a Riemann sum.

Let $q_j$ denote the probability density function of the uniform noise over $A_j$:
\begin{equation}
  \forall j \in \left\{ 1,\dots,m\right\}, z \in \mathbb{R}^d, q_j(z) = \frac{1}{V_n}\mathbb{1}_{z \in A_j}
\end{equation}

Let $V = \Vol( B^d_{\infty}(0,r)) $. For $j \in \left\{ 1, \dots, m\right\}$, let $p_{j} = \int h(z)q_{j}(z)dz$ be the expected response of the true classifier $h$ to noise $q_j$ (i.e. what is observed), and the coefficients $k_j$ such that:
\begin{equation}
  k_j = \begin{cases}
  \frac{V_n}{V} & \text{if $p_j=1$, \ie, $h=1$ almost surely on $A_j$}\\
  0 & \text{otherwise.}
  \end{cases}
\end{equation}

We choose these specific coefficients to only "activate" the squares where $h=1$ almost surely, i.e. that are entirely inside of the decision region.

Let $\delta$ be any attack vector of norm $\epsilon$, and $\tilde{q_0} = q_0(\ \cdot\ - \delta)$ be the distribution after attack by $\delta$. The support of $\tilde{q_0}$ is $B_\infty^d(- \delta, r)$ which is fully contained in $\bigcup\limits_{i=1}^{m}A_i$.

Let $\Phi_\mathcal{K}$ be the Neyman-Pearson function defined by the $\mathcal{K}:=\{ k_1, \dots, k_n\}$:
\begin{equation}
  \Phi_\mathcal{K}(x') = 
  \begin{cases}
  1 &\text{if $\tilde{q_0}(x') \le \sum_{i=1}^{n} k_i q_i(x')$}\\
  0 &\text{otherwise}\\
  \end{cases} 
\end{equation}

We know that $\Phi_\mathcal{K} = 1$ outside of $B_\infty^d(x + \delta, r)$ since $\tilde{q_0}=0$ there.
Let $x' \in B_\infty^d(x+\delta, r)$.
The $A_i$ are disjoint and cover the ball, so there is exactly one $j$ such that $x' \in A_j$.
We then have for any $x'\in A_j$:
\begin{equation}
  \Phi_\mathcal{K}(x') = 
  \begin{cases}
  1 &\text{if $h=1$ almost surely on $A_j$}\\
  0 &\text{otherwise}
  \end{cases} 
\end{equation}

Hence $\Phi_\mathcal{K}|_{A_j} = \essinf\limits_{A_j}(h)$, since $h$ has values in $\{0,1\}$. It follows that:
\begin{equation}
\int \Phi_\mathcal{K}(z) \tilde{q_0}(z)dz = \sum\limits_{i=1}^{m} (\essinf\limits_{A_j}(h))\Vol(A_i \cap B_\infty^d(x-\delta,r))
\end{equation}

That is a lower Riemann sum for the integral $\int\limits_{B_\infty^d(x-\delta,r)} h$, and so converges to it when $m\to\infty$ as $h$ is Riemann integrable.
Hence we can choose n such that, for any $\delta$,
\begin{equation}
  \int \Phi_\mathcal{K}(z) \tilde{q_0}(z)dz \leq \int h(z) \tilde{q_0}(z)dz + \xi
\end{equation}
which gives us the desired result, since this is true for any $\delta$.
\end{proof}

\section{Proofs of \Cref{section:experiments}}

\subsection{Proof of \Cref{theorem:different_sigma}}
\differensigma*

\begin{proof}[Proof of \Cref{theorem:different_sigma}]

Let $x \in \mathbb{R}^d$. $q_0$ is an isotropic probability density function, which means that there exists a function $p_0$ such that $\forall x \in \mathbb{R}^d, q_0(x) = p_0(\| x \|^2)$. For $i=1\dots n$, we have:

\begin{equation}
  \label{equation:proof_diff_sigma_certificate}
  \Pbb\left[ \Ncal(x,\sigma_i^2) \in \Scal_{k_1,\dots,k_n}\right] = \frac{1}{(2\pi)^{\frac{d}{2}} \sigma_0^d} \int\limits_{\Rbb^d} \exp\left( -\frac{\norm{u - x- \delta}^2}{2\sigma_0^2} \right) \mathbb{1}_{u \in \Scal_{k_1,\dots,k_n}} \diff u
\end{equation}
where $\Scal$ is the Neyman-Person set defined as:
\begin{equation}
  \Scal_{k_1,\dots,k_n} = \left\{ u \in \Rbb^d \mathrel{\bigg|} p_0(\norm{ u - x - \delta}^2) \leq \sum_{i = 0}^n k_i \exp\left(-\frac{\norm{u-x}^2}{2\sigma_i^2}\right) \right\}
\end{equation}

Where the $k_i$ are defined as in~\Cref{corollary:NPcertificate}.

By expressing~\Cref{equation:proof_diff_sigma_certificate} with hypercylindrical coordinates of center x and axis $\delta$, we have:
\begin{align}
  \Pbb\left[ \Ncal(x,\sigma_i^2) \in \Scal_{k_1,\dots,k_n}\right] &= \frac{1}{(2\pi)^{\frac{d}{2}} \sigma_i^d} \int\limits_{\Rbb^d} \exp\left( -\frac{\norm{u-x}^2}{2\sigma_i^2} \right) \mathbb{1}_{x \in \Scal} \diff x \\
  &= \frac{1}{(2\pi)^{\frac{d}{2}} \sigma_0^d} \int\limits_{\substack{\mu \in \Rbb \\ r \in \Rbb^+ \\ \phi_1, \dots, \phi_{d-3} \in [-\pi, \pi] \\ \phi_{d-2} \in [0, 2\pi]}}  \exp\left( - \frac{r^2 + \mu^2}{2\sigma^2} \right) \mathbb{1}_{(r, \mu) \in \tilde{\Scal}} J \diff r \diff \mu \diff \phi_{1} \dots \diff \phi_{d-2}
\end{align}
where from~\citet{blumenson1960derivation}, the Jacobian $J$ of the change of variables is:
\begin{equation}
    J = r^{d-2} \prod_{k = 1}^{d-2} \sin^k \phi_{d-1-k}
\end{equation}
and where $\tilde{\Scal}$ is the updated Neyman-Person set:
\begin{equation}
  \tilde{\Scal} = \left\{ r, \mu \in \Rbb \mathrel{\bigg|} p_0(r^2 + (\mu - \epsilon)^2) \leq \sum_{i = 0}^n k_i \exp\left(-\frac{r^2+\mu^2}{2\sigma_i^2}\right) \right\}
\end{equation}

Given that the indicator function is independent of the $\phi_1, \dots, \phi_{d-2}$, we can rearrange the above equation as follows:
\begin{align}
  \Pbb\left[ \Ncal(x,\sigma_i^2) \in \Scal_{k_1,\dots,k_n}\right] &= \frac{1}{(2\pi)^{\frac{d}{2}} \sigma_i^d} \left(\ \ \int\limits_{\substack{\mu \in \Rbb \\ r \in \Rbb^+}}  \exp\left( - \frac{r^2 + \mu^2}{2\sigma^2} \right) r^{d-2} \mathbb{1}_{(r, \mu) \in \tilde{\Scal}}  \diff r \diff \mu \right) \notag \\
  &\quad\quad\quad\quad \left( \int\limits_{\substack{\phi_1, \dots, \phi_{d-3} \in [-\pi, \pi] \\ \phi_{d-2} \in [0, 2\pi]}} \prod_{k = 1}^{d-2} \sin^k \phi_{d-1-k} \diff \phi_{1} \dots \diff \phi_{d-2} \right)
\end{align}
By setting $A$ as:
\begin{equation}
    A = \int\limits_{\substack{\phi_1, \dots, \phi_{d-3} \in [-\pi, \pi] \\ \phi_{d-2} \in [0, 2\pi]}} \prod_{k = 1}^{d-2} \sin^k \phi_{d-1-k} \diff \phi_{1} \dots \diff \phi_{d-2}
\end{equation}
we have:
\begin{equation}
  \Pbb\left[ \Ncal(x,\sigma_i^2) \in \Scal_{k_1,\dots,k_n}\right] = \frac{A}{(2\pi)^{\frac{d}{2}} \sigma_0^{d}} \left(\ \  \int\limits_{\mu \in \Rbb} \int\limits_{r \in \Rbb^+} \exp\left(-\frac{r^2}{2\sigma_0^2} \right) \exp\left(-\frac{\mu^2}{2\sigma^2} \right) r^{d-2} \mathbb{1}_{(r, \mu) \in \tilde{\Scal}} \diff r \diff \mu \ \ \right)
\end{equation}
Finally, we can express this probability with an expected value over a Gaussian and Chi distribution:
\begin{equation}
    \Pbb\left[ \Ncal(x,\sigma_i^2) \in \Scal_{k_1,\dots,k_n}\right] = \Ebb_{\substack{\mu \sim \Ncal(0, \sigma_i^2) \\ r \sim \chi(d-1, 0, \sigma_i^2)}} \left[ \mathbb{1}_{(r, \mu) \in \tilde{\Scal}} \right]  
\end{equation}
which concludes the proof.
\end{proof}

\end{document}